%% file: main.tex
\documentclass[10pt]{article} 
\usepackage[accepted]{tmlr}

\input{macros}

\input{macros2}

\usepackage{hyperref}
\usepackage{url}

\usepackage{tikz}
\usepackage{caption}
\usepackage{subcaption}
\usepackage{graphicx}
\usepackage{titletoc}
\usepackage{xfrac}

\usepackage{amsmath}
\usepackage{amssymb}
\usepackage[capitalize,nameinlink,noabbrev]{cleveref}
\crefformat{equation}{(#2#1#3)}
\Crefformat{equation}{#2Equation~#1#3}
\crefrangeformat{equation}{(#3#1#4) to~(#5#2#6)}
\crefname{assumption}{Assumption}{Assumptions}

\usepackage[shortlabels]{enumitem}
\usepackage{url}

\title{Adversarial Surrogate Risk Bounds \\for Binary Classification}

\author{\name Natalie Frank \email natalief@uw.edu \\
      \addr Department of Applied Mathematics\\
      University of Washington
      }

\def\month{MM}  
\def\year{YYYY} 
\def\openreview{\url{https://openreview.net/forum?id=XXXX}} 

\begin{document}

\maketitle

\input{Sections/0-abstract}

\input{section_list}

\bibliography{bibliography,bib2}
\bibliographystyle{tmlr}

\newpage
\appendix
\section*{Contents of Appendix}
\addcontentsline{toc}{section}{Appendix}
\startcontents[appendix]
\printcontents[appendix]{l}{1}{\setcounter{tocdepth}{2}}
\newpage
\input{appendix_list}

\end{document}

%% file: macros.tex
\usepackage{amsfonts}
\usepackage{mathrsfs}
\usepackage{amsmath}
\usepackage{color}
\definecolor{blue}{rgb}{0.0, 0.0, 1.0}

\definecolor{jgGreen}{rgb}{0.0, 0.5, 0.0}

\definecolor{pink}{rgb}{0.5, 0.0, 0.25}

\newcommand{\ignore}[1]{}

\definecolor{darkorange}{RGB}{255, 140, 0}

\newcommand{\prm}{R_\phi^\e}
\newcommand{\dl}{\bar R_\phi}

\newcommand{\hingephi}{{\phi_{\text{hinge}}}}
\newcommand{\cprm}{R^\e}
\newcommand{\cdl}{\bar R}
\newcommand{\const}[1]{\frac 1 {\phi(0)-C_\phi^*(1/2-#1)}}

\newcommand{\cB}{{\mathcal B}}

\newcommand{\cH}{\mathcal{H}}

\newcommand{\bx}{{\mathbf x}}
\newcommand{\by}{{\mathbf y}}
\newcommand{\bz}{{\mathbf z}}

\DeclareMathOperator*{\esssup}{ess\,sup}

\DeclareMathOperator{\id}{id}

\DeclareMathOperator{\supp}{supp}

\DeclareMathOperator{\conc}{conc}

\DeclareMathOperator{\dom}{dom}

\def\Rset{\mathbb{R}}

\newcommand{\PP}{{\mathbb P}}
\newcommand{\QQ}{{\mathbb Q}}

\newcommand{\one}{{\mathbf{1}}}
\newcommand{\zero}{{\mathbf{0}}}
\newcommand{\Wball}[1]{{\cB^\infty_{#1}}}

\newcommand{\ov}{\overline}

\newcommand{\td}{\tilde}

\newcommand{\e}{\epsilon}

%% file: macros2.tex
\usepackage{amsthm}

\newtheorem{assumption}{Assumption}
\newtheorem{theorem}{Theorem}
\newtheorem{definition}{Definition}
\newtheorem{lemma}{Lemma}

\newtheorem{proposition}{Proposition}
\theoremstyle{definition}
\newtheorem{example}{Example}

%% file: Sections/0-abstract.tex
\begin{abstract}%
A central concern in classification is the vulnerability of machine learning models to adversarial attacks. Adversarial training is one of the most popular techniques for training robust classifiers, which involves minimizing an adversarial surrogate risk.  Recent work has characterized the conditions under which any sequence minimizing the adversarial surrogate risk also minimizes the adversarial classification risk in the binary setting, a property known as \emph{adversarial consistency}. However, these results do not address the rate at which the adversarial classification risk approaches its optimal value along such a sequence. This paper provides surrogate risk bounds that quantify that convergence rate.  
\end{abstract}%

%% file: section_list.tex
\input{Sections/1-Introduction}
\input{Sections/2-Background}
\input{Sections/3-main_results}

\input{Sections/4-non-convex_adv_surrogate_bound}
\input{Sections/5-convex_consistent_losses}

\input{Sections/6-non-consistent_bound}

\input{Sections/7-Related_Works}
\input{Sections/8-conclusion}

\input{Sections/9-acks}

%% file: Sections/1-Introduction.tex
\section{Introduction}
A central concern regarding regarding sophisticated machine learning models is their susceptibility to adversarial attacks. Prior work \citep{biggio2013evasion,szegedy2013intriguing} demonstrated that imperceptible perturbations can degrade the performance of neural nets. As such models are deployed in security-critical applications, including facial recognition \citep{Xu2022face} and medical imaging \citep{PaschaliConjetiNavarroNavab2018medical}, training robust models remains a key challenge in machine learning.

In the standard classification setting, the \emph{classification risk} is the proportion of incorrectly classified data. Directly minimizing this quantity is a combinatorial optimization problem, so typical machine learning algorithms instead minimize a more tractable \emph{surrogate} risk via gradient-based methods. A surrogate risk is said to be consistent for a given data distribution if every minimizing sequence also minimizes the classification risk for that distribution. Beyond consistency, a central objective is efficiency: minimizing the surrogate risk should translate into a rapid reduction of the classification risk. This rate can be quantified via surrogate risk bounds, which bound the excess classification risk in terms of the excess surrogate risk.

In the standard binary classification setting, consistency and surrogate risk bounds are well-studied topics \citep{BartlettJordanMcAuliffe2006,Lin2004,Steinwart2007,zhang04}. A typical approach reduces the problem to a pointwise analysis of the conditional classification and surrogate risks.
In contrast, the adversarial setting is less understood. The adversarial classification risk penalizes instances that can be perturbed into the opposite class, while the adversarial surrogate risk computes the worst-case value over an $\epsilon$-ball. The dependence on the value of a function over an $\e$-ball precludes a pointwise decomposition, rendering the classical analysis inapplicable. 
\Citet{FrankNilesWeed23consistency} characterized the risks that are consistent for all data distributions, and the corresponding losses are referred to as \emph{adversarially consistent}. Unfortunately, no convex loss function can be adversarially consistent for all data distributions \citep{MeunierEttedguietal22}. On the other hand, \citet{frank2024consistency} showed that such situations are rather atypical--- when the data distribution is absolutely continuous, a surrogate risk is adversarially consistent so long as the adversarial Bayes classifier satisfies a certain notion of uniqueness. While these results characterize consistency, none describe convergence rates.

\textbf{Our Contributions:} 
\begin{itemize}
    \item We prove a linear surrogate risk bound for adversarially consistent losses (\cref{thm:adv_surrogate_massart}).
    \item When the “distribution of optimal attacks” satisfies a bounded noise condition, we prove a linear surrogate risk bound under mild conditions on the loss (\cref{thm:adv_surrogate_massart}).
    \item We establish a distribution-dependent surrogate risk bound that applies whenever a loss is adversarially consistent for the data distribution (\cref{thm:convex_surrogate_bound}). 
\end{itemize}
    Notably, the last result applies to convex loss functions. By prior consistency results \citep{frank2024consistency,FrankNilesWeed23consistency,MeunierEttedguietal22}, one cannot hope for distribution independent surrogate bounds for non-adversarially consistent losses. This work presents a framework for surrogate risk bounds that applies to any supremum-based risk under mild conditions. A detailed comparison with prior work is provided in \cref{sec:related_works}. 

%% file: Sections/2-Background.tex
\section{Background and Preliminaries}
\subsection{Surrogate Risks}
We study binary classification on $\mathbb{R}^d$ with labels $-1$ and $+1$, where $\PP_0$ and $\PP_1$ denote the class-conditional distributions. For a measurable set $A$, the classification risk is  
\[
R(A) = \int 1_{A^C} \, d\PP_1 + \int 1_A \, d\PP_0,
\]
with minimum $R^*$ over all Borel sets. Because the indicator function is nondifferentiable, one instead minimizes a \emph{surrogate risk}  
\[
R_\phi(f) = \int \phi(f) \, d\PP_1 + \int \phi(-f) \, d\PP_0,
\]
with minimum $R_{\phi,*}$ over all Borel functions. The loss $\phi$ satisfies:

\begin{assumption}\label{as:phi}
$\phi$ is continuous, non-increasing, and $\lim_{\alpha \to \infty} \phi(\alpha) = 0$.
\end{assumption}

Thresholding $f$ at zero yields the classifier $\{f>0\}$, whose risk is  
\[
R(f)=R(\{f>0\}) = \int 1_{f \le 0} \, d\PP_1 + \int 1_{f > 0} \, d\PP_0.
\]

It remains to verify that minimizing the surrogate risk $R_\phi$ will also minimize the classification risk $R$. 
\begin{definition}
    The loss function $\phi$ is \emph{consistent for the distribution $\PP_0$, $\PP_1$} if every minimizing sequence of $R_\phi$ is also a minimizing sequence of $R$. The loss function $\phi$ is \emph{consistent} if it is consistent for all distributions.   
\end{definition}
Prior work establishes conditions under which many common loss functions are consistent. For convex $\phi$, consistency occurs iff $\phi$ is differentiable at $0$ and $\phi'(0) < 0$ \citep[Theorem~2]{BartlettJordanMcAuliffe2006}. \citet[Proposition~3]{FrankNilesWeed23consistency} show that consistency holds if $\inf_\alpha \frac{1}{2}(\phi(\alpha)+\phi(-\alpha)) < \phi(0)$, which is satisfied by losses such as
the $\rho$-margin loss $\phi_\rho(\alpha)=\min(1,\max(1-\alpha/\rho,0))$ and the shifted sigmoid loss $\phi_\tau(\alpha)=1/(1+\exp(\alpha-\tau))$, $\tau>0$. However, a convex loss $\phi$ cannot satisfy this inequality:
\begin{equation}\label{eq:adversarially_consistent_non_convex}
    \frac 12\left(\phi(\alpha)+\phi(-\alpha)\right)\geq\phi\big( \frac 12 \alpha +\frac 12 \cdot -\alpha\big)= \phi(0).
\end{equation}
\subsection{Surrogate Risk Bounds}
In addition to consistency, quantifying convergence rates is a key concern. Specifically, prior work \citep{BartlettJordanMcAuliffe2006,zhang04} establishes \emph{surrogate risk bounds} of the form $\Psi(R(f)-R_*)\leq R_\phi(f)-R_{\phi,*}$ for some function $\Psi$, linking excess classification risk to excess surrogate risk. These bounds involve pointwise minima of the \emph{conditional} classification and surrogate risks.

Let $\PP=\PP_0+\PP_1$ and $\eta(\bx)=d\PP_1/d\PP$. An equivalent formulation of the classification risk is 
\begin{equation}\label{eq:classification_risk_eta}
    R(f)=\int C(\eta(\bx),f(\bx))d\PP(\bx)    
\end{equation}
where $C(\eta, \alpha)=\eta \one_{\alpha\leq 0}+(1-\eta)\one_{\alpha>0}$,
with minimal conditional risk 
\begin{equation}\label{eq:C^*_def}
    C^*(\eta)=\inf_\alpha C(\eta,\alpha)=\min(\eta,1-\eta),
\end{equation}
and thus the minimal classification risk is $R_*=\int C^*(\eta(\bx))d\PP(\bx)$.
Analogously, the surrogate risk in terms of $\eta$ and $\PP$ is 
\begin{equation}\label{eq:surrogate_pw}
    R_\phi(f)= \int C_\phi(\eta(\bx),f(\bx))d\PP, \quad C_\phi(\eta,\alpha)=\eta\phi(\alpha)+(1-\eta)\phi(-\alpha)
\end{equation}
and the minimal surrogate risk is 
$R_{\phi,*}=\int C_\phi^*(\eta(\bx))d\PP(\bx)$
with the  minimal conditional risk $C_\phi^*(\eta)$ defined by 
\begin{equation}\label{eq:C_phi^*_def}
\quad C_\phi^*(\eta)=\inf_\alpha C_\phi(\eta,\alpha).    
\end{equation}

Prior work on consistency typically establishes surrogate risk bounds via pointwise analysis of the conditional risks, relating the excess conditional surrogate risk $C_\phi(\eta,\alpha) - C_\phi^*(\eta)$ to the excess conditional classification risk $C(\eta,\alpha) - C^*(\eta)$.

 The consistency of $\phi$ can be fully characterized by the properties of the function $C_\phi^*(\eta)$.

\begin{theorem}\label{th:consistent_characterization}
        A loss $\phi$ is consistent iff $C_\phi^*(\eta)<\phi(0)$ for all $\eta\neq 1/2$.
    \end{theorem}
Surprisingly, this criterion has not appeared in prior work. See \cref{app:consistent_characterization} for a proof.
In terms of the function $C_\phi^*$, \citep[Proposition~3]{FrankNilesWeed23consistency} states that any loss $\phi$ with $C_\phi^*(1/2)<\phi(0)$ is consistent. 
The function $C_\phi^*$ is a key component of surrogate risk bounds from prior work. Specifically, \citet{BartlettJordanMcAuliffe2006} show:

\begin{theorem}[\citep{TewariBartlett2007}]\label{prop:psi_def}
        Let $\phi$ be any loss satisfying \cref{as:phi} with $C_\phi^*(1/2)=\phi(0)$ and define
    
    \[\Psi(\theta)=\phi(0)-C_\phi^*\left (\frac {1+\theta} 2\right).\]
        
        Then
        \begin{equation}\label{eq:C_bound}
            \Psi(C(\eta,f)-C^*(\eta))\leq C_\phi(\eta,f)-C_\phi^*(\eta)
        \end{equation}
        and consequently
        \begin{equation}\label{eq:R_bound}\Psi(R(f)-R_*)\leq R_\phi(f)-R_\phi^*.
        \end{equation}
        
    \end{theorem}
    The inequality \cref{eq:R_bound} is a consequence of \cref{eq:C_bound} and Jensen's inequality. \Cref{th:consistent_characterization} implies that this bound is non-vacuous iff $\phi$ is consistent--- compare with \cref{th:consistent_characterization}. Moreover, \cref{eq:C_bound} yields a distribution-dependent linear surrogate bound when $\eta$ is bounded away from $1/2$.
    If \emph{Massart's noise condition}  \citep{Massart06} holds--- namely, there exists a $\alpha\in[0,1/2]$ for which $|\eta-1/2|\geq \alpha$ $\PP$-a.e., then the distribution admits a linear surrogate bound.
    \begin{proposition}\label{prop:surrogate_massart}
    Let $\eta$, $\PP$ be a distribution that satisfies $|\eta-1/2|\geq \alpha$ $\PP$-a.e. with a constant $\alpha \in [0,1/2]$, and let $\phi$ be a loss with $\phi(0)>C_\phi^*(1/2-\alpha)$. Then for all $|\eta-1/2|\geq \alpha$,
    \begin{equation}\label{eq:massart_loss}
        C(\eta,f)-C^*(\eta)\leq \frac {1}{\phi(0)-C_\phi^*(\frac 12 -\alpha)}(C_\phi(\eta,f)-C_\phi^*(\eta))    
    \end{equation}
    and consequently
    \begin{equation}\label{eq:massart_surrogate}
        R(f)-R_*\leq \frac {1}{\phi(0)-C_\phi^*(\frac 12 -\alpha)} (R_\phi(f)-R_{\phi,*})    
    \end{equation}
        
    \end{proposition}
    
 See \cref{app:surrogate_linear} for a proof of this result. Observe that \cref{th:consistent_characterization} guarantees that the linear constant is finite whenever $\alpha\neq 0$ and $\phi$ is consistent. This bound is distribution-independent when $\phi(0)>C_\phi^*(1/2)$ with $\alpha=0$, and will later be generalized to adversarial risks. Although the constant in \cref{prop:surrogate_massart} is not optimal, further refinement offers no improvement to our adversarial bounds, so we opt to retain the simpler form.

\subsection{Adversarial Risks}
     The adversarial classification risk incurs a penalty of 1 whenever a point $\bx$ can be perturbed into the opposite class. This penalty can be expressed in terms of supremums of indicator functions--- the adversarial classification risk incurs a penalty of 1 whenever $\sup_{\|\bx'-\bx\|\leq \e} \one_A(\bx')=1$ or $\sup_{\|\bx'-\bx\|\leq \e} \one_{A^C}(\bx')=1$. Define
    \[S_\e(g)(\bx)=\sup_{\|\bx-\bx'\|\leq \e} g(\bx').\] The adversarial classification and surrogate risks are given respectively by\footnote{In order to define the risks $\prm$ and $\cprm$, one must argue that $S_\e(g)$ is measurable. Theorem~1 of \citep{FrankNilesWeed23minimax} proves that whenever $g$ is Borel, $S_\e(g)$ is always measurable with respect to the completion of any Borel measure.} 
    \[\cprm(A)=\int S_\e(\one_{A^C})d\PP_1+\int S_\e(\one_A)d\PP_0,\quad \prm(f)=\int S_\e(\phi(f))d\PP_1+\int S_\e(\phi(-f))d\PP_0.\]
    A minimizer of the adversarial classification risk is called an \emph{adversarial Bayes classifier}. After optimizing the surrogate risk, a classifier is obtained by thresholding the resulting function $f$ at zero. The associated adversarial classification error function $f$ is then
    \begin{equation}\label{eq:adv_classification_risk_alt}
        \cprm(f)=\cprm(\{f>0\})=\int S_\e(\one_{f\leq 0})d\PP_1+\int S_\e(\one_{f>0})d\PP_0.    
    \end{equation}
    
    Just as in the standard case, one would hope that minimizing the adversarial surrogate risk would minimize the adversarial classification risk. 
    \begin{definition}  
        The loss $\phi$ is \emph{adversarially consistent for the distribution $\PP_0$, $\PP_1$} if any minimizing sequence of $\prm$ is also a minimizing sequence of $\cprm$. 
        We say that $\phi$ is \emph{adversarially consistent} if it is adversarially consistent for all distributions. 
    \end{definition}
        Theorem~2 of \citep{FrankNilesWeed23consistency} characterizes the adversarially consistent losses:
    \begin{theorem}[\citep{FrankNilesWeed23consistency}]\label{th:adv_consistency}
        The loss $\phi$ is adversarially consistent iff $C_\phi^*(1/2)<\phi(0)$.
    \end{theorem}
\Citep[Proposition~3]{FrankNilesWeed23consistency} guarantees that every adversarially consistent loss is also consistent in the standard sense. Unfortunately, \cref{eq:adversarially_consistent_non_convex} shows that no convex loss is adversarially consistent. However, distributions for which consistency fails are atypical: for absolutely continuous $\PP$, adversarial consistency holds provided the adversarial Bayes classifier is \emph{unique up to degeneracy}.

\begin{definition}
    Two adversarial Bayes classifiers $A_1$, $A_2$ are \emph{equivalent up to degeneracy} if any set $A$ with $A_1 \cap A_2 \subset A \subset A_1 \cup A_2$ is also an adversarial Bayes classifier. The adversarial Bayes classifier is \emph{unique up to degeneracy} if any two adversarial Bayes classifiers are equivalent up to degeneracy.
\end{definition}
See \cref{fig:adv_bayes_classifiers} for an illustration of non-equivalent adversarial Bayes classifiers in a distribution where adversarial consistency fails. 
Theorem~4 of \citep{frank2024consistency} relates uniqueness of the adversarial Bayes classifier to the consistency of $\phi$.
    \begin{theorem}[\citep{frank2024consistency}]\label{th:adv_consistency_and_uniqueness}
        Let $\phi$ be a loss with $C_\phi^*(1/2)=\phi(0)$ and assume that $\PP$ is absolutely continuous with respect to Lebesgue measure. Then $\phi$ is adversarially consistent for the data distribution given by $\PP_0$, $\PP_1$ iff the adversarial Bayes classifier is unique up to degeneracy.
    \end{theorem}

\begin{figure}
    \centering
    \begin{subfigure}{0.45\textwidth}
        \centering
        \begin{tikzpicture}
            \draw[dashed] (0,0) circle (1);
            \fill (0,0) circle (0.05);
            \draw[<->] (0.05,0) -- (0.95,0) node[midway, above] {$\epsilon$};
        \end{tikzpicture}
        \caption{\label{subfig:a}}
    \end{subfigure}
    \begin{subfigure}{0.45\textwidth}
        \centering
        \begin{tikzpicture}
            \fill[blue!20] (-1.5,-1.5) rectangle (1.5,1.5);
            \fill[white] circle (1);
            \draw[dashed] (0,0) circle (1);
            \fill (0,0) circle (0.05);
            \draw[<->] (0.05,0) -- (0.95,0) node[midway, above] {$\epsilon$};
        \end{tikzpicture}
        \caption{\label{subfig:b}}
    \end{subfigure}
    
    \begin{subfigure}{0.45\textwidth}
        \centering
        \begin{tikzpicture}
            \fill[blue!20] circle (1);
            \draw[dashed] (0,0) circle (1);
            \fill (0,0) circle (0.05);
            \draw[<->] (0.05,0) -- (0.95,0) node[midway, above] {$\epsilon$};
        \end{tikzpicture}
        \caption{\label{subfig:c}}
    \end{subfigure}
    \begin{subfigure}{0.45\textwidth}
        \centering
        \begin{tikzpicture}
            \fill[blue!20] (-1.5,-1.5) rectangle (1.5,1.5);
            \draw[dashed] (0,0) circle (1);
            \fill (0,0) circle (0.05);
            \draw[<->] (0.05,0) -- (0.95,0) node[midway, above] {$\epsilon$};
        \end{tikzpicture}
        \caption{\label{subfig:d}}
    \end{subfigure}
    
    \caption{\label{fig:adv_bayes_classifiers}Adversarial Bayes classifiers for the distribution where $\PP_0=\PP_1$ are uniform distributions on $\ov{B_\e(\zero)}$, the counterexample from \citet{MeunierEttedguietal22}. The classifiers in (a) and (b) are equivalent up to degeneracy, as are those in (c) and (d), but the classifiers in (a) and (c) are not. A sequence minimizing $\prm$ but not $\cprm$ is provided in \cref{eq:f_n}.}
\end{figure}
    
    Any extension of surrogate risk bounds to the adversarial setting must account for the conditions of \cref{th:adv_consistency,th:adv_consistency_and_uniqueness}.
\subsection{Minimax Theorems}
A central tool in analyzing the adversarial consistency of surrogate risks is minimax theorems, which enable a `pointwise'-style representation of adversarial risks analogous \cref{eq:surrogate_pw}. This section reviews the minimax representation for both adversarial classification and surrogate risks, which underlie the bounds in \cref{sec:main_results}.

These minimax theorems utilize the $\infty$-Wasserstein ($W_\infty$) metric from optimal transport. Informally, this metric quantifies the smallest radius $\e$ such that the mass of one distribution can be transported to match another without moving any point more than $\e$. 

Formally, let $\QQ$ and $\QQ'$ be finite positive measures with equal total mass. A Borel measure $\gamma$ on $\Rset^d\times \Rset^d$ is a coupling between $\QQ$ and $\QQ'$ if its first marginal is $\QQ$ and its second marginal is $\QQ'$, or in other words, $\gamma(A\times \Rset^d)=\QQ(A)$ and $\gamma(\Rset^d\times A)=\QQ'(A)$ for all Borel sets $A$. Denote the set of couplings between $\QQ$ and $\QQ'$ by $\Pi(\QQ, \QQ')$. Then the $W_\infty$ distance is
\begin{equation}\label{eq:W_inf_def}
    W_\infty(\QQ,\QQ')=\inf_{\gamma\in \Pi(\QQ,\QQ')} \esssup_{(\bx,\by)\sim \gamma} \|\bx-\by\|.
\end{equation}
Theorem~2.6 of \citep{Jylha15} proves that the infimum in \cref{eq:W_inf_def} is always attained. The $\e$-ball around $\QQ$ in the $W_\infty$ metric is 
$\Wball \e(\QQ)=\{\QQ':W_\infty(\QQ',\QQ)\leq \e\}$.

The next lemma is a standard observation linking adversarial perturbations to $W_\infty$-balls. We include a proof in \cref{app:S_e_and_W_inf} for completeness; it is a known result and not new to this work (see for instance \citep[Proposition~3.1]{StaibJegelka2017}).

\begin{lemma}\label{lemma:S_e_and_W_inf}
     Let $g$ be a Borel function. Let $\gamma$ be a coupling between the measures $\QQ$ and $\QQ'$ supported on $\Delta_\e=\{(\bx,\bx'):\|\bx-\bx'\|\leq \e\}$. Then $S_\e(g)(\bx)\geq g(\bx')$ $\gamma$-a.e. and consequently
     \[\int S_\e(g)d\QQ\geq \sup_{\QQ'\in \Wball \e(\QQ) } \int g d\QQ'.\]
 \end{lemma}

 Applying \cref{lemma:S_e_and_W_inf} to $\cprm$ shows that $\inf_A \cprm(A)$ can be expressed as an inf-sup problem. The minimax theorem of \citep{PydiJog2021} ensures that the order of the $\inf$ and $\sup$ can be interchanged. Let $C^*(\eta)$ be as defined in \cref{eq:C^*_def} and define

\begin{equation}\label{eq:cdl_def}
    \cdl(\PP_0',\PP_1')=\inf_{A\text{ Borel}} \int \one_{A^C}d\PP_1'+\int \one_Ad\PP_0'=\int C^*\left( \frac{d\PP_1'}{d\left(\PP_1'+\PP_0'\right)}\right)d\left(\PP_0'+\PP_1' \right).    
\end{equation}

\begin{theorem}[\citep{frank2024consistency}]\label{th:strong_duality_classification}  Let $\cdl$ be as defined in \cref{eq:cdl_def}. Then 

    \[\inf_{A\text{ Borel}} \cprm(A)=\sup_{\substack{\PP_1'\in \Wball \e(\PP_1)\\\PP_0'\in\Wball \e(\PP_0)}}\cdl(\PP_0',\PP_1').\]
    with equality attained at some Borel  $A$, $\PP_0^*\in \Wball \e (\PP_0)$, and $\PP_1^*\in \Wball \e(\PP_1)$.  
\end{theorem}
See \citep[Theorem~1]{FrankNilesWeed23consistency} for a proof of this statement. The maximizers $\PP_0^*$, $\PP_1^*$ can be interpreted as optimal adversarial attacks (see discussion following \citep[Theorem~7]{FrankNilesWeed23minimax}).
\citet[Theorem~3.4]{frank2024uniquness} provide a criterion for uniqueness up to degeneracy in terms of dual maximizers.
\begin{theorem}[\citep{frank2024consistency}]\label{th:TFAE_unique_up_to_den}
   The following are equivalent:
   \begin{enumerate}[label=\Alph*)]
       \item The adversarial Bayes classifier is unique up to degeneracy
       \item There are maximizers $\PP_0^*$, $\PP_1^*$ of $\cdl$ for which 
       $\PP^*(\eta^*=1/2)=0$, where $\PP^*=\PP_0^*+\PP_1^*$ 
       and $\eta^*=d\PP_1^*/d\PP^*$
   \end{enumerate}
\end{theorem}
Thus, uniqueness corresponds to the situation in which the set where both classes are equally probable has measure zero under some optimal adversarial attack.

 The analogous dual problem to $\prm$ uses $C_\phi^*(\eta)$ from \cref{eq:C_phi^*_def}

\begin{equation}\label{eq:dl_def}
    \dl(\PP_0',\PP_1')=\inf_{f\text{ Borel}} \int \phi(f)d\PP_1'+\int \phi(-f)d\PP_0'=\int C_\phi^*\left( \frac{d\PP_1'}{d\left(\PP_1'+\PP_0'\right)}\right)d\left(\PP_0'+\PP_1' \right)    
\end{equation}
and the analogous minimax theorem states (\citet[Theorem~6]{FrankNilesWeed23minimax}): 
\begin{theorem}[\citep{FrankNilesWeed23minimax}]\label{th:strong_duality_surrogate}
Let $\dl$ be defined as in \cref{eq:dl_def}. Then
   \[\inf_{\substack{f\text{ Borel,}\\  \Rset\text{-valued}}} \prm(f)=\sup_{\substack{\PP_1'\in \Wball \e(\PP_1)\\\PP_0'\in\Wball \e(\PP_0)}}\dl(\PP_0',\PP_1').\]
    with maximizers $\PP_0^*\in \Wball \e (\PP_0)$, $\PP_1^*\in \Wball \e(\PP_1)$ attained.
    
\end{theorem}

Finally, optimal attacks for the surrogate problem are also optimal for the classification problem:
\begin{theorem}\label{th:maximizer_comparison}
    Consider maximizing the dual objectives $\dl$ and $\cdl$ over $\Wball \e(\PP_0)\times \Wball \e(\PP_1)$.
    \begin{enumerate}[label=\arabic*)]
    \item   If $\phi$ is consistent, then any maximizer $(\PP_0^*, \PP_1^*)$ of $\dl$ over $\Wball \e(\PP_0)\times \Wball \e(\PP_1)$ also maximizes $\cdl$. \label{it:maximizers_match_1}
    \item {[\citep{frank2024consistency}]} If the adversarial Bayes classifier is unique up to degeneracy, then there exists a maximizer $(\PP_0^*,\PP_1^*)$ of $\dl$ with $\PP^*(\eta^*=1/2)=0$, where $\PP^*=\PP_0^*+\PP_1^*$ and $\eta^*=d\PP_1^*/d\PP^*$.\label{it:maximizers_match_2}
    \end{enumerate}

\end{theorem}
See \cref{app:maximizer_comparison} for a proof of \cref{it:maximizers_match_1}, \cref{it:maximizers_match_2} is shown in Theorems~5 and~7 of \citep{frank2024consistency}.
This minimax machinery links the adversarial Bayes classifier, optimal attacks, and surrogate risks, establishing the dual formulations used in Section 3 to derive adversarial surrogate risk bounds.

%% file: Sections/3-main_results.tex
\section{Main Results}\label{sec:main_results}

Prior work has characterized when a loss $\phi$ is adversarially consistent with respect to a distribution $\PP_0$, $\PP_1$. \Cref{th:adv_consistency} shows that a distribution-independent surrogate risk bound is possible only when $C_\phi^*(1/2)<\phi(0)$. When $C_\phi^*(1/2)=\phi(0)$, \cref{th:adv_consistency_and_uniqueness} indicates that any such bound must depend on the marginal distribution of $\eta^*$ under $\PP^*$, and moreover, is possible only if $\PP^*(\eta^*=1/2)=0$. 

Compare these statements with \cref{prop:surrogate_massart}: \cref{th:adv_consistency,th:adv_consistency_and_uniqueness,th:maximizer_comparison} together imply if either $C_\phi^*(1/2)<\phi(0)$ or if there exist some maximizers of $\dl$ that satisfy Massart's noise condition, then $\phi$ is adversarially consistent for $\PP_0$, $\PP_1$. Alternatively, due to \cref{th:maximizer_comparison}, one can equivalently assume that there are maximizers of $\dl$ satisfying Massart's noise condition. Our first result extends \cref{prop:surrogate_massart} to the adversarial scenario, replacing $\PP_0$, $\PP_1$ with the distribution of optimal adversarial attacks.

\begin{theorem}\label{thm:adv_surrogate_massart} \label{thm:non_convex_adversarial_surrogate_bound}Let $\phi$ be consistent and let $\PP_0$, $\PP_1$ be a distribution for which there are maximizers $\PP_0^*$, $\PP_1^*$ of the dual problem $\dl$ that satisfy $|\eta^*-1/2|\geq \alpha$ $\PP^*$-a.e. for some constant $\alpha\in [0,1/2]$ with $C_\phi^*(1/2-\alpha)<\phi(0)$, where $\PP^*=\PP_0^*+\PP_1^*$, $\eta^*=d\PP_1^*/d\PP^*$.
     Then 
     \begin{equation}\label{eq:adv_surrogate_massart_ineq}
        R^\e(f)-R_*^\e\leq  \const \alpha \left(R_\phi^\e(f)-R_{\phi,*}^\e\right)     
     \end{equation}
    
\end{theorem}

When $C_\phi^*(1/2)<\phi(0)$, setting $\alpha=0$ in \cref{thm:adv_surrogate_massart} yields a distribution-independent bound. As noted earlier, two losses satisfying this condition are the $\rho$-margin loss and the shifted sigmoid loss. Likewise, \cref{th:consistent_characterization} ensures that the linear constant is finite whenever $\alpha\neq 0$ and $\phi$ is consistent. 

The constant appearing in \cref{thm:adv_surrogate_massart} is nearly optimal: \Cref{sec:lower_bound} shows that it can be improved by at most a factor of two, and this gap is attained by a known counterexample to consistency. Thus, the result provides a sharp characterization of how tightly the adversarial classification risk can be controlled by the surrogate risk across all consistent convex losses.

Furthermore, the theorem parallels the classical realizable-case guarantee from the non-adversarial setting. If the optimal adversarial risk satisfies $R_*^\e = 0$, then Massart’s noise condition holds with $\alpha = 1/2$ (see \cref{lemma:realizable_eta}). In this regime, \cref{thm:adv_surrogate_massart} yields a linear relationship between adversarial classification and surrogate risks that is directly analogous to the non-adversarial bound in \cref{prop:surrogate_massart}.
Zero adversarial risk occurs whenever the supports of $\PP_0$ and $\PP_1$ are separated by at least $2\epsilon$ (\cref{ex:realizable,fig:realizable}).

\Cref{thm:adv_surrogate_massart} states that if some distribution of \emph{optimal adversarial attacks} satisfies Massart's noise condition, then the excess adversarial surrogate risk is at worst a linear upper bound on the excess adversarial classification risk. However, if $C_\phi^*(1/2)=\phi(0)$, the bound's constant diverges as $\alpha\to 0$, reflecting the failure of adversarial consistency when the adversarial Bayes classifier is not unique up to degeneracy. For $\alpha \neq 1/2$, understanding the assumptions on $(\PP_0,\PP_1)$ which ensure Massart’s condition for the distribution of adversarial attacks $(\PP_0^*,\PP_1^*)$ remains an open problem. Example~4.6 of \citep{frank2024uniquness} exhibits a distribution that satisfies Massart's noise condition and yet the adversarial Bayes classifier is not unique up to degeneracy. Thus Massart's noise condition for $\PP_0,\PP_1$ does not guarantee Massart's noise condition for $\PP_0^*$, $\PP_1^*$. See \cref{ex:massart,fig:massart} for an example where \cref{thm:adv_surrogate_massart} applies with $\alpha>0$.

One approach to relaxing the distributional restriction is to apply \cref{eq:adv_surrogate_massart_ineq} only on the portion of the distribution where $|\eta^*-1/2|\geq \alpha$ and then add back in the risk on  $|\eta^*-1/2|< \alpha$.

\begin{theorem}\label{thm:non_consistent_bound_linear}
    Assume that there exist maximizers $\mathbb P_0^\ast$, $\mathbb P_1^\ast$ of $\bar R_\phi$ that are induced by transport maps from $\PP_0$, $\PP_1$, and define $\PP^*=\PP_1^*+\PP_0^*$, $\eta^*=d\PP_1^*/d\PP^*$. Let $0\leq \alpha$, then
    \[\cprm(f)-\cprm_*\leq \const \alpha\big(\prm(f)-R^\e_{\phi,*}\big) +\left(\frac 12 +\alpha\right)\PP^*(|\eta^*-1/2|<\alpha)\]
\end{theorem}
Since this holds for all $\alpha$, the right-hand side can be minimized over $\alpha$. 
Prior work from optimal transport theory verifies the assumption on $\mathbb P^\ast$ under mild conditions: Theorem~3.5 of \citep{Jylha15} states that whenever $\mathbb P_0,\PP_1$ are absolutely continuous with respect to Lebesgue measure and the norm $\|\cdot\|$ is strictly convex, the measures $\mathbb P_0^\ast,\PP_1^\ast$ are induced by a transport map. It is unclear whether this holds for common datasets such as CIFAR-10 or MNIST.

Finally, an alternative approach to removing the distributional restriction is to average bounds of the form \cref{eq:adv_surrogate_massart_ineq} over all values of $\eta^*$ yielding a distribution-dependent surrogate bound, valid whenever the adversarial Bayes classifier is unique up to degeneracy.
For a given function $f$, let the \emph{concave envelope} of $f$ be the smallest concave function larger than $f$:
\begin{equation}\label{eq:conc_def}
    \conc(f)=\inf\{g:\geq f\text{ on }\dom(f), g\text{ concave and upper semi-continuous}\}    
\end{equation}

\begin{theorem}\label{thm:convex_surrogate_bound}
    Assume $\PP_0(\Rset^d)+\PP_1(\Rset^d)\leq 1$, $\phi$ is a consistent loss with $C_\phi^*(1/2)=\phi(0)$, and the adversarial Bayes classifier is unique up to degeneracy. Let $\PP_0^*$, $\PP_1^*$ be maximizers of $\dl$ for which $\PP^*(\eta^*=1/2)=0$, with $\PP^*=\PP_0^*+\PP_1^*$ and $\eta^*=d\PP_1^*/d\PP^*$. Define $H(z)=\conc(\PP^*(|\eta^*-1/2|\leq z))$, $\Psi$ as \cref{prop:psi_def}, and let $\tilde \Lambda(z)=\Psi^{-1}(\min(\frac z 4,\phi(0)))$.
    Then 
     \[\cprm(f)-R^\e_*\leq \tilde \Phi(\prm(f) -R_{\phi,*}^\e)\]
    with 
    \[\tilde \Phi(z)=4\left(\id+\min(1,\sqrt{-{e}H\ln H })\right)\circ \tilde\Lambda \]
\end{theorem}
This theorem is established under the assumption $\PP_0(\Rset^d)+\PP_1(\Rset^d)\leq 1$, which serves as an essential intermediate step for extending the result to case where the adversarial Bayes classifier is not uniquely defined up to degeneracy.
See \cref{ex:gaussian,fig:gaussian} for an example of calculating a distribution-dependent surrogate risk bound.

    The function $H$ is always continuous and satisfies $H(0)=0$, ensuring that this bound is non-vacuous (see \cref{lemma:envelope} in \cref{sec:convex_surrogate_bound}). Further notice that $H\ln H$ approaches zero as $H\to 0$.

 The map $\tilde \Phi$ combines two components: $\tilde \Lambda$, a modified version of $\Psi^{-1}$, and $H$, a modification of the cdf of $|\eta^*-1/2|$. The function $\tilde \Lambda$ is a scaled version of $\Psi^{-1}$, where $\Psi$ is the surrogate risk bound in the non-adversarial case of \cref{prop:psi_def}. The domain of $\Psi^{-1}$ is $[0,\phi(0)]$, and thus the role of the $\min$ in the definition of $\tilde \Lambda$ is to truncate the argument so that it fits into this domain. The factor of $1/4$ in this function appears to be an artifact of our proof, see \cref{sec:convex_surrogate_bound} for further discussion. In contrast, the map $H$ translates the distribution of $\eta^*$ into a surrogate risk transformation. Compare with \cref{th:adv_consistency_and_uniqueness}, which states that consistency fails if $\PP^*(\eta^*=1/2)>0$; accordingly, the function $H$ becomes a poorer bound when more mass of $\eta^*$ is near $1/2$.

If $\PP^*(\eta^*=1/2)$ is small, this result can still provide an informative surrogate bound.

\begin{theorem} \label{thm:non_consistent_bound_nonlinear}
        Assume that there exist maximizers $\mathbb P_0^\ast$, $\mathbb P_1^\ast$ of $\bar R_\phi$ that are induced by transport maps from $\PP_0^*$, $\PP_1^*$, and define $\PP^*=\PP_1^*+\PP_0^*$, $\eta^*=d\PP_1^*/d\PP^*$.  Let $\tilde \Phi$ be the function in \cref{thm:convex_surrogate_bound}, but with $H$ defined as $H(z)=\text{conc}(\mathbb P^\ast(0<|\eta^\ast-1/2|\leq z))$. Then 
     $$ R^\epsilon (f)-R_\ast^\epsilon\leq \tilde \Phi(R_\phi^\epsilon(f)-R^\epsilon_{\phi,\ast}) +\frac {\PP^*(\eta^*=1/2)} 2 $$
\end{theorem}
Removing the assumption that $\PP_0^\ast,\PP_1^*$ are induced by a transport map from \cref{thm:non_consistent_bound_linear,thm:non_consistent_bound_nonlinear} remains an open problem. We conjecture that this assumption is, in fact, unnecessary.

\subsection*{Comparison with real-world datasets}
\begin{figure}
	\centering
	\begin{subfigure}[b]{0.225\textwidth}
		\centering
		\includegraphics[width=\textwidth]{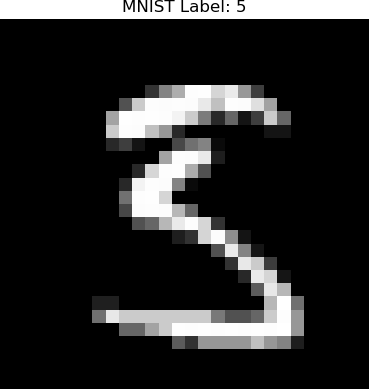}
		\caption{}
        \label{fig:hard_mnist5_1}
	\end{subfigure}
	\hfill
	\begin{subfigure}[b]{0.225\textwidth}
		\centering
		\includegraphics[width=\textwidth]{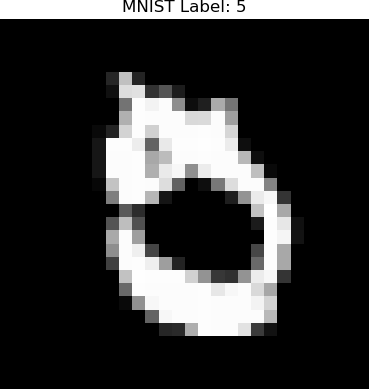}
		\caption{}
        \label{fig:hard_mnist5_2}
	\end{subfigure}
	\hfill
	\begin{subfigure}[b]{0.225\textwidth}
		\centering
		\includegraphics[width=\textwidth]{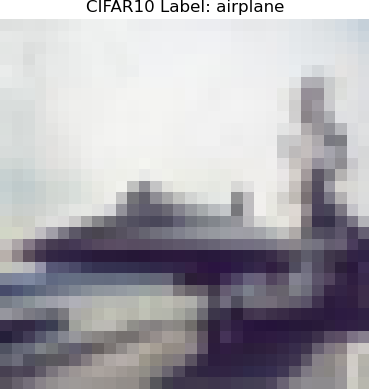}
		\caption{}
        \label{fig:hard_cifar_plane_1}
	\end{subfigure}
	\hfill
    	\begin{subfigure}[b]{0.225\textwidth}
		\centering
		\includegraphics[width=\textwidth]{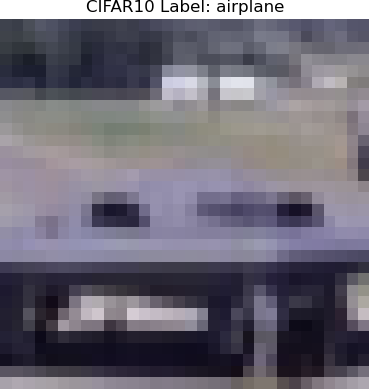}
		\caption{}
        \label{fig:hard_cifar_plane_2}
	\end{subfigure}
    \caption{Ambiguous images in the MNIST and CIFAR10 datasets. (a) lies between a `5' and a `3', while (b) is difficult to classify at all, despite being labeled as a `5'. In CIFAR10, image (c) is ambiguous between a ship and an airplane, and image (d) is similarly hard to identify.}
	\label{fig:ambiguous images}
\end{figure}
Experimental results from prior work suggest that, in real-world datasets, $\eta^\ast$ is typically concentrated near 0 and 1. \citet{BhagojiCullinaMittalji2019lower} compute lower bounds on the adversarial classification risk for binary tasks, focusing on classifying digits `3' and `7' in MNIST under $\ell_2$ perturbations. Their lower bound remains close to 0 for $\epsilon \leq 3$ and increases to 0.2 at $\epsilon = 4$. Since $C^\ast(\eta^\ast)$ attains its maximum at $\eta^\ast = 1/2$, a small adversarial risk implies that the distribution places little mass in a neighborhood of $|\eta^\ast - 1/2| = 0$. Similar trends are observed on Fashion MNIST and CIFAR10. \Citet{DaiDingBhagoji2023characterizing} extend these bounds to the multiclass setting, though extending adversarial surrogate bounds beyond binary classification remains an open problem.

When the optimal adversarial risk is non-zero, the adversarial Bayes classifier may not be unique up to degeneracy. Even without adversarial perturbations, datasets like MNIST and CIFAR10 contain inherently ambiguous examples. \Citet{NorthcuttAthalyeMeuller2021labelerrors} identify such cases, four are depicted in \cref{fig:ambiguous images}. One would expect $\eta(\bx) = 1/2$ for such examples. \Citet{BartoldsonDiffenderKailkhura2024scaling} show that similar ambiguity arises in adversarial settings: under $\ell_\infty$ perturbations of size $8/255$, approximately 6\% of adversarial examples are ambiguous in the CIFAR10 dataset. In the binary scenario, one would thus expect $\eta^\ast(x) = 1/2$ for these inputs, and thus one must apply \cref{thm:non_consistent_bound_linear} or \cref{thm:non_consistent_bound_nonlinear}. Extending the concept of uniqueness of the adversarial Bayes classifier to multiclass settings remains an open problem.

\subsection*{Examples}

Below we present three examples illustrating the applicability of our main theorems. All examples involve one-dimensional distributions, and we denote the pdfs of $\PP_0$ and $\PP_1$ by $p_0$ and $p_1$.

To start, if $R_{*}^\e=0$ then $\eta^*\in\{0,1\}$ $\PP^*$-a.e. for any maximizers of $\dl$. Therefore, for any such distribution, the optimal attack satisfies Massart's noise condition with $\alpha=1/2$, see \cref{app:realizable_eta} for a proof.
\begin{lemma}\label{lemma:realizable_eta}
    Assume $R_*^\e=0$, let $(\PP_0^*,\PP_1^*)$ maximize $\dl$, and define $\PP^*=\PP_0^*+\PP_1^*$, $\eta^*=d\PP_1^*/d\PP^*$. Then $\PP^*(\eta^*\in\{0,1\})=1$.
\end{lemma}

Any distribution for which the supports of $\PP_0$, $\PP_1$ are more than $2\e$ apart must have zero risk. 

\begin{example}[When $R_*^\e=0$]\label{ex:realizable} Let 
\[p_0(x)=\begin{cases}
1&\text{if } x\in[-1-\delta,-\delta]\\
0&\text{otherwise}
\end{cases}\quad p_1(x)=\begin{cases} 1&\text{if }x\in [\delta,1+\delta]\\
0&\text{otherwise}
\end{cases}\]
for some $\delta>0$. See \cref{fig:realizable} for a depiction. This distribution satisfies $R^\e_*=0$ for all $\e\leq \delta$ and thus \cref{lemma:realizable_eta} implies that the surrogate bound of \cref{thm:adv_surrogate_massart} applies. 
    
\end{example}

\begin{figure}
     \centering
     \begin{subfigure}[b]{0.30\textwidth}
         \centering
         \includegraphics[width=\textwidth]{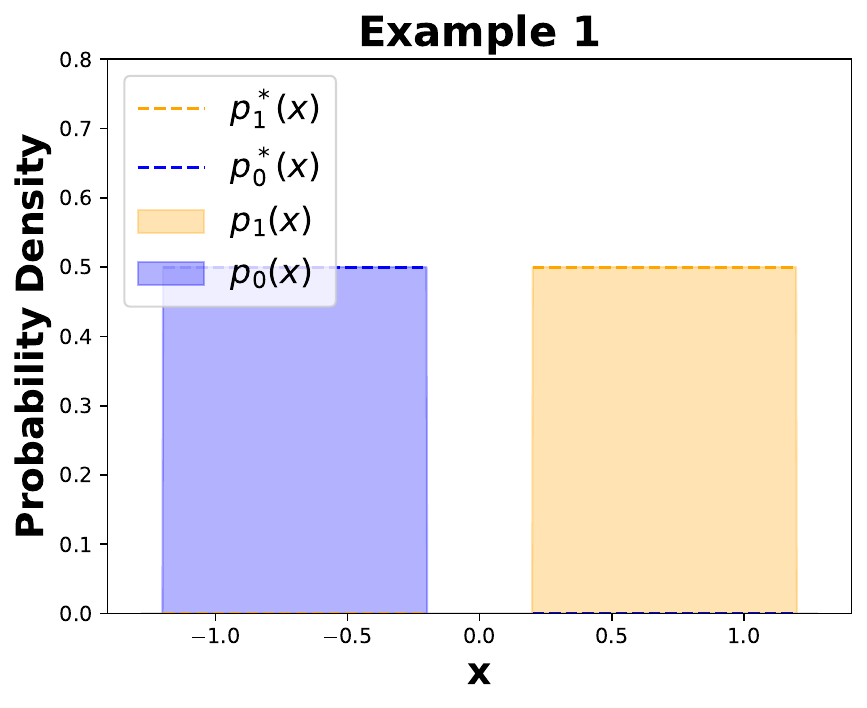}
         \caption{}
         \label{fig:realizable}
     \end{subfigure}
     \hfill
     \begin{subfigure}[b]{0.30\textwidth}
         \centering
         \includegraphics[width=\textwidth]{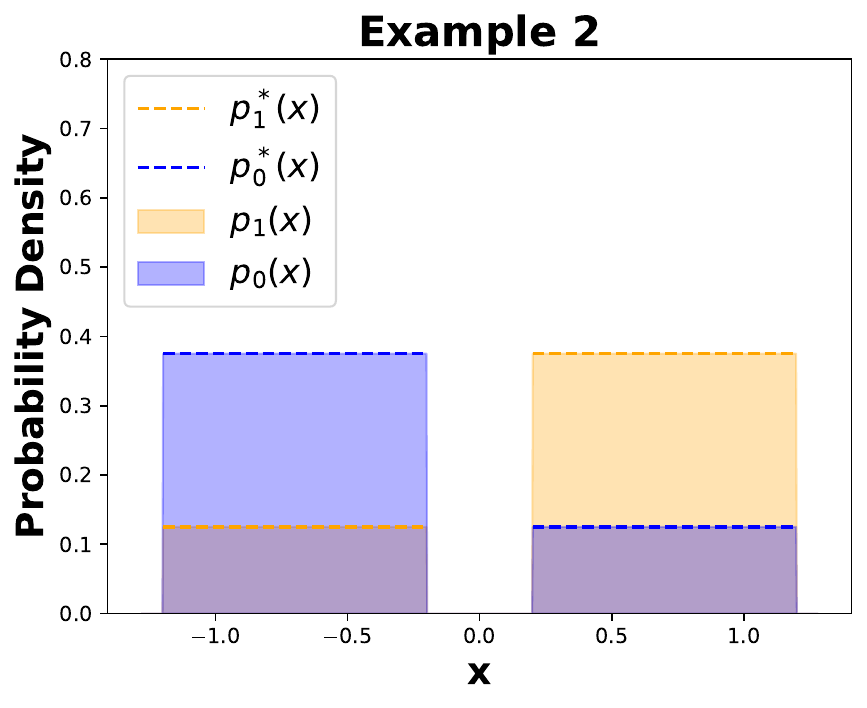}
         \caption{}
         \label{fig:massart}
     \end{subfigure}
      \hfill
     \begin{subfigure}[b]{0.30\textwidth}
         \centering
        \includegraphics[width=\textwidth]{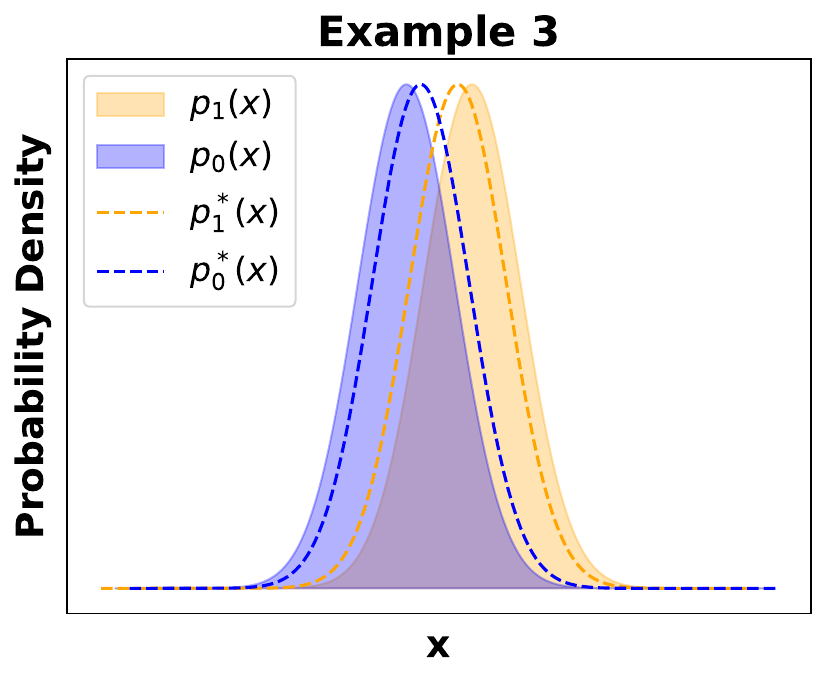}
         \caption{}
         \label{fig:gaussian}
     \end{subfigure}\caption{Distributions from \cref{ex:realizable,ex:massart,ex:gaussian} along with attacks that maximize the dual $\dl$.}
\end{figure}


\Cref{ex:massart,ex:gaussian} require computing maximizers to the dual $\dl$; See \cref{app:dual_attack,app:dual_attack_massart} for these calculations. The following example illustrates a distribution for which Massart's noise condition can be verified for a distribution of optimal attacks.
\begin{example}[Massart's noise condition]\label{ex:massart}
Let $\delta>0$ and let $p$ be the uniform density on $[-1-\delta, -\delta]\cup [\delta,1+\delta]$. Define $\eta$ by
\begin{equation}\label{eq:eta_massart_example}
    \eta(x)=\begin{cases}
    \frac 14& \text{if }x \in [-1-\delta,-\delta]\\
    \frac 34&\text{if }x\in [\delta,1+\delta]
\end{cases}    
\end{equation}

see \cref{fig:massart} for a depiction of $p_0$ and $p_1$. For this distribution and $\e\leq \delta$, the minimal surrogate and adversarial surrogate risks are always equal ($R_{\phi,*}=R^\e_{\phi,*}$). This fact together with \cref{th:strong_duality_surrogate} imply that optimal attacks on this distribution are $\PP_1^*=\PP_1$ and $\PP_0^*=\PP_0$, see \cref{app:dual_attack_massart} for details. Consequently: the distribution of optimal attacks $\PP_0^*$, $\PP_1^*$ satisfies Massart's noise condition with $\alpha=1/4$ and as a result the bounds of \cref{thm:adv_surrogate_massart} apply. When $\e \in (\delta,1+\delta)$, pdfs of the distributions that maximize the dual are $p_1^*(x)=p_1(x+\e)$, $p_0^*(x)=p_0(x-\e)$, where $p_1(x)=\eta(x)p(x)$ and $p_0(x)=(1-\eta(x))p(x)$. These distributions satisfy $\PP^*(\eta=1/2)=(\e-\delta)$ while $\PP^*(|\eta-1/2|\geq 1/4)=1-(\e-\delta)$. Thus \cref{thm:non_consistent_bound_linear} provides a surrogate bound. \end{example}

The final example presents a case in which Massart’s noise condition fails for the distribution of optimal adversarial attacks, yet the adversarial Bayes classifier remains unique up to degeneracy. \Cref{thm:convex_surrogate_bound} still yields an informative surrogate bound.

\begin{example}[Gaussian example] \label{ex:gaussian}
    Consider an equal-variance Gaussian mixture with $\mu_0+2\e < \mu_1<\mu_0+\sqrt 2 \sigma$:
    \[p_0(x)=\frac 12 \cdot\frac 1 {\sqrt{2\pi}\sigma} e^{-\frac{(x-\mu_0)^2}{2\sigma^2}},\quad p_1(x)=\frac 12\cdot \frac 1 {\sqrt{2\pi}\sigma}e^{-\frac{(x-\mu_1)^2}{2\sigma^2}},\]
    see \cref{fig:gaussian} for a depiction. The optimal attacks $\PP_0^*$, $\PP_1^*$ are gaussians centered at $\mu_0+\e$ and $\mu_1-\e$ respecively, with pdfs
    \begin{equation}
        \label{eq:dual_optimal_attacks}
        p_0^*(x)=\frac 12 \cdot\frac 1 {\sqrt{2\pi}\sigma} e^{-\frac{(x-(\mu_0+\e))^2}{2\sigma^2}},\quad p_1^*(x)=\frac 12\cdot \frac 1 {\sqrt{2\pi}\sigma}e^{-\frac{(x-(\mu_1-\e))^2}{2\sigma^2}}.
    \end{equation}
    
We verify that $\PP_0^*$ and $\PP_1^*$ are in fact optimal by finding a function $f^*$ for which $\prm(f^*)=\dl(\PP_0^*,\PP_1^*)$, the strong duality result in \cref{th:strong_duality_surrogate} will then imply that $\PP_0^*$ and $\PP_1^*$ must maximize the dual $\dl$, see \cref{app:dual_attack} for details.

    Further, when $ \mu_1-\mu_0\leq \sqrt 2 \sigma$, then the function $h(z)=\PP^*(|\eta^*-1/2|\leq z)$ is concave in $z$ and consequently $H=h$, see \cref{app:gaussian_linear_bound} for details. Although $h$ is unwieldy function, comparison to its linear approximation at zero gives the bound
    \begin{equation}
    \label{eq:concavity_bound}
        H(z)\leq \min\left( \frac {16 \sigma^2}{\mu_1-\mu_0-2\e}z,1\right).
    \end{equation}
    Again, see \cref{app:gaussian_linear_bound} for details.
\end{example}

When $\e\geq (\mu_1-\mu_0)/2$, \citep[Example~4.1]{frank2024uniquness} demonstrates that the adversarial Bayes classifier is not unique up to degeneracy. Notably, the bound in preceding example deteriorates as $(\mu_1-\mu_0)/2 \to \e$, and then fails entirely when $\e=(\mu_1-\mu_0)/2$.

%% file: Sections/4-non-convex_adv_surrogate_bound.tex
\section{Proof of Linear Surrogate Bounds}\label{sec:linear}

\subsection{Proof of \cref{thm:adv_surrogate_massart}}

The proof of Theorem~9 relies on decomposing the excess adversarial classification and surrogate risks into non-negative terms, revealing their structural similarity and allowing for a pointwise comparison.

Let $\PP_0^*,\PP_1^*$ be any maximizers of $\dl$. These distributions also maximize $\cdl$ by \cref{th:maximizer_comparison}. Set $\PP^*=\PP_0^*+\PP_1^*$, $\eta^*=d\PP_1^*/d\PP^*$. Let $\gamma_0^*$, $\gamma_1^*$ be couplings between $\PP_0$, $\PP_0^*$ and $\PP_1$, $\PP_1^*$ achieving the $W_\infty$ distances \cref{eq:W_inf_def}. The excess classification risk can be decomposed as
\begin{align}
    &\cprm(f)-R_*^\e= \cprm(f)-\cdl(\PP_0^*,\PP_1^*)=\int i_1(f)d\gamma_1^*+\int i_0(f)d\gamma_0^*\label{eq:division_classification}
\end{align}
with 
\[i_1(f)=\big( S_\e(\one_{f\leq 0})(\bx)- \one_{f\leq 0}(\bx')  \big)+\big(  C(\eta^*,f)-C^*(\eta^*)\big)\]
\[i_0(f)=\big(S_\e(\one_{f>0})(\bx) - \one_{f>0}(\bx')\big)+\big( C(\eta^*,f)-C^*(\eta^*) \big). \]

The first term measures the discrepancy between the worst-case attack on $f$
 and the attack induced by $\PP_0^*,\PP_1^*$, the optimal attack for the distribution $\PP_0,\PP_1$. The second term measures the excess conditional risk under the optimal attack $\PP_0^*,\PP_1^*$. 
\Cref{lemma:S_e_and_W_inf} implies that $S_\e(\one_{f\leq 0})(\bx)- \one_{f\leq 0}(\bx')$ must be positive, while the definition of $C^*$ implies that $C(\eta^*,f)-C^*(\eta^*)\geq 0$.

Similarly, one can express the excess surrogate risk as 
\begin{equation}\label{eq:division_surrogate}
\begin{aligned}
    \prm(f)-R_{\phi,*}^\e&=\int i_1^\phi(f)d\gamma_1^*+\int i_0^\phi(f)d\gamma_0^*
\end{aligned}
\end{equation}
with 
\[i_1^\phi(f)=\big( S_\e(\phi(f))(\bx) -\phi(f)(\bx')\big)+\big( C_\phi(\eta^*,f)-C_\phi^*(\eta^*) \big)\]
\[i_0^\phi(f)=\big( S_\e(\phi(-f))(\bx)-\phi(-f)(\bx')\big)+\big( C_\phi(\eta^*,f)-C_\phi^*(\eta^*) \big)
\]

 The following lemma is the core inequality linking $i_k$ to $i_k^\phi$ under Massart's noise condition. It shows that each classification-risk term can be bounded by a constant multiple of its surrogate risk counterpart.

        \begin{lemma}\label{lemma:linear_adv_target}
        Define $i_0^\phi,i_1^\phi$ as in \eqref{eq:division_surrogate} and assume that the distribution of optimal adversarial attacks $\PP_0^*$, $\PP_1^*$ satisfies Massart's noise condition. Then

            \noindent\begin{minipage}{0.5\linewidth}
        \begin{equation}\label{eq:non-convex_I_0_bound}
        i_0(f)\leq \const \alpha i_0^\phi(f).
    \end{equation}
    \end{minipage}%
    \begin{minipage}{0.5\linewidth}
        \begin{equation}\label{eq:non-convex_I_1_bound}
        i_1(f)\leq \const \alpha i_1^\phi(f).
    \end{equation}
    \end{minipage}
    hold $\gamma_0^*$-a.e. and $\gamma_1^*$-a.e. respectively.
    \end{lemma}

    \Cref{lemma:linear_adv_target} directly implies \cref{thm:adv_surrogate_massart} by integration over couplings $\gamma_1^*$, $\gamma_0^*$.

    \begin{proof}[Proof of \cref{thm:adv_surrogate_massart}]
        Combine \cref{eq:division_classification}, \cref{lemma:linear_adv_target}, and \cref{eq:division_surrogate}.
    \end{proof}
    
    \subsection{Proof of \cref{lemma:linear_adv_target}}

    The proof proceeds by partitioning the domain $\Rset^d\times \Rset^d$ into regions where the supremum-based classification either matches ($D_k$) or exceeds ($E_k$) the decision under the optimal attack. On each region, we derive a separate bound relating $i_k$ and $i_k^\phi$. 
      Define the sets $D_k$, $E_k$, 
        \begin{align}
            &D_0=\{(\bx,\bx'): S_\e(\one_{f> 0})(\bx)-\one_{f(\bx')}=0\}\label{eq:D_0_orig}\\
            &E_0=\{(\bx,\bx'): S_\e(\one_{f> 0})(\bx)-\one_{f(\bx')> 0}=1\}\label{eq:E_0_orig}\\
            &D_1=\{(\bx,\bx'): S_\e(\one_{f\leq 0})(\bx)-\one_{f(\bx')\leq 0}=0\}\label{eq:D_1_orig}\\
            &E_1=\{(\bx,\bx'): S_\e(\one_{f\leq 0})(\bx)-\one_{f(\bx')\leq 0}=1\}\label{eq:E_1_orig}
        \end{align}
    
    By construction, $D_1\cup E_1=\Rset^d\times \Rset^d$ and $D_0\cup E_0=\Rset^d\times \Rset^d$.

    The following lemma records a simple but useful structural property of $E_0$ and $E_1$, which allows us to bound the surrrogate loss terms from below on these sets.
    \begin{lemma}\label{lemma:DEF_structure}
        Let $E_k$ be as in \cref{eq:E_1_orig,eq:E_0_orig}. Then $S_\e(\one_{f>0})(\bx)=\one_{f>0}(\bx')=1$ $\gamma_1^*$-a.e. on $E_1$ while $S_\e(\one_{f\leq 0})(\bx)=\one_{f\leq 0}(\bx')=1$ $\gamma_0^*$-a.e. on $E_0$.
    \end{lemma}
    \begin{proof}
        We'll prove the statement for $E_1$, the argument for $E_0$ is analogous. Specifically, we will show that one cannot simultaneously have $ S_\e(\one_{f\leq 0})(\bx)-\one_{f\leq 0}(\bx')=1$ and $S_\e(\one_{f> 0})(\bx)-\one_{f> 0}(\bx')=1$. 
        
        Consider $(\bx,\bx')\in E_1$: as both $S_\e(\one_{f\leq 0})(\bx)$ and $\one_{f\leq0}(\bx')$ are 0-1 valued, the relation $S_\e(\one_{f\leq 0})(\bx)-\one_{f(\bx')\leq 0}=1$ implies that $\one_{f(\bx')\leq 0}=0$ and thus $\one_{f(\bx')>0}=1$. The fact that $S_\e(\one_{f> 0})(\bx)\geq \one_{f> 0}(\bx')$ on $\supp \gamma_1^*$ and $\supp \gamma_1^* \subset \Delta_\e$ implies that $S_\e(\one_{f> 0})(\bx)=1$ $\gamma_1^*$-a.e. on $E_1$.

    \end{proof}
    The next result bounds the terms $i_k^\phi$ from below.
    \begin{lemma}\label{lemma:DEF_surrogate_lower_bound}
    The relations \cref{eq:i_0_lower_bound} and \cref{eq:i_1_lower_bound} hold on $E_0$ and $E_1$ respectively.
    
    \noindent\begin{minipage}{0.5\linewidth}
                \begin{equation}\label{eq:i_0_lower_bound}
        i_0^\phi(f)\geq \phi(0)-C_\phi^*(\eta^*).
    \end{equation}
    \end{minipage}%
    \begin{minipage}{0.5\linewidth}

            \begin{equation}\label{eq:i_1_lower_bound}
        i_1^\phi(f)\geq \phi(0)-C_\phi^*(\eta^*).
    \end{equation}
    \end{minipage}
 
    \end{lemma}
    \begin{proof}
        We will prove the statement for $E_1$, the argument for $E_0$ is analogous.
        Observe that 
        \[i_1^\phi(f)=S_\e(\phi(f))(\bx)+(1-\eta^*)(\phi(-f(\bx'))-\phi(f(\bx')))-C_\phi^*(\eta^*)\]
        Now as $S_\e(\one_{f\leq 0})(\bx)=1$, one can conclude that there is a point in $\bz\in\ov{B_\e(\bx)}$ for which $f(\bz)\leq 0$, and thus $S_\e(\phi(f))(\bx)\geq \phi(0)$. Next, \cref{lemma:DEF_structure} implies that $f(\bx')>0$ and hence $\phi(-f(\bx'))-\phi(f(\bx'))\geq 0$. Therefore, one can conclude \cref{eq:i_1_lower_bound}.

    \end{proof}
        Furthermore, a simple calculation bounds the $i_k$ from above.
        \begin{lemma}\label{lemma:DEF_class_upper_bound}
        On the set $D_k$
        \begin{equation}\label{eq:D_class_upper_bound}
            i_k(f)=C(\eta^*,f)-C^*(\eta^*)    
        \end{equation}

    while on $E_k$
    \begin{equation}\label{eq:E_class_upper_bound}
        i_k(f)=1+C(\eta^*,f)-C^*(\eta^*)
    \end{equation}
    
    \end{lemma}
    \begin{proof}
            We will show the statement $k=1$ the argument for $k=0$ is analogous.
            On $D_1$, $S_\e(\one_{f\leq 0})(\bx)-\one_{f(\bx')\leq 0}=0$, implying \cref{eq:D_class_upper_bound}. Similarly,     
            on $E_1$, $S_\e(\one_{f\leq 0})(\bx)-\one_{f(\bx')\leq 0}=1$, implying \cref{eq:E_class_upper_bound}.
            
    \end{proof}

    Comparing the upper and lower bounds present in \cref{lemma:DEF_structure,lemma:DEF_class_upper_bound} proves \cref{lemma:linear_adv_target}. 
    
        \begin{proof}[Proof of \cref{lemma:linear_adv_target}]
        We will discuss \cref{eq:non-convex_I_1_bound}, the argument for \cref{eq:non-convex_I_0_bound} is analogous. We prove the bound separately on $D_1$ and $E_1$, whose union is $\Rset^d$.
         First, notice that \cref{eq:massart_loss} implies that
    \begin{equation}\label{eq:adv_attack_bound_1}
        C(\eta^*(\bx'),f(\bx'))-C^*(\eta^*(\bx'))\leq \frac 1 {\phi(0)-C_\phi^*(1/2-\alpha)}\left(C_\phi(\eta^*(\bx'),f(\bx'))-C_\phi^*(\eta^*(\bx'))\right)\quad \PP^*\text{-a.e.}
    \end{equation}

                    \textbf{On the set $D_1$:}\\
            \Cref{lemma:DEF_class_upper_bound} implies that 
            \begin{align*}
                & i_1(f)=C(\eta^*(\bx'),f(\bx')) -C^*(\eta^*(\bx')) 
            \end{align*}
            and thus the desired inequality follows from \cref{eq:adv_attack_bound_1} and the fact that $S_\e(\phi\circ f)(\bx)-\phi\circ f(\bx')\geq 0$ $\gamma_1^*$-a.e.

        \textbf{On the set $E_1$:}\\
        On $E_1$,
        \[i_1(f)=1+C(\eta^*(\bx'),f(\bx')) -C^*(\eta^*(\bx'))\]
        However, due to \cref{lemma:DEF_surrogate_lower_bound},
    \begin{equation}\label{eq:non_convex_reasoning_transfer_1}
        \begin{aligned}
            S_\e(\one_{f\leq 0})(\bx)-\one_{f\leq 0}(\bx')&=1=\frac{ \phi(0)-C_\phi^*(\eta^*)}{ \phi(0)-C_\phi^*(\eta^*)}\leq  \frac{ 1}{ \phi(0)-C_\phi^*(\eta^*)} (S_\e(\phi\circ f)(\bx)-\phi(f(\bx'))
        \end{aligned}
    \end{equation}
    The last inequality is a consequence of the assumption $|\eta^*-1/2|\leq \alpha$.
    Summing this relation with \cref{eq:adv_attack_bound_1} shows \cref{eq:non-convex_I_1_bound}.
    \end{proof}

\subsection{Lower Bounds}\label{sec:lower_bound}
The bound in \Cref{thm:adv_surrogate_massart} provides a general guarantee relating the adversarial classification and surrogate risks.  
We now show that this bound cannot be substantially improved.  

This example describes functions in which the worst-case attack and the attack induced by $\PP_0^*$, $\PP_1^*$ differ substantially. This function sequence is the counterexample to consistency proposed in prior work \citep{MeunierEttedguietal22,LiTelgarsky2023achieving}. Intuitively, the sign flip causes the classifier to misclassify both classes, even though a constant function would achieve lower risk for this distribution.
\begin{example}[Lower bound for \cref{thm:adv_surrogate_massart}]
    Let $\phi$ satisfy $C_\phi^\ast(1/2) = \phi(0)$, and consider a distribution supported on $[-\epsilon, \epsilon]$ with $\mathbb{P}^\ast(\eta = 1/2 + \alpha) = 1$. Define the sequence of functions 
        \begin{equation}\label{eq:f_n}
        f_n=\begin{cases}
        \frac 1n & \bx \neq 0\\
        -\frac 1n &\bx =0
    \end{cases}
    \end{equation}
    For this sequence, $R^\epsilon(f_n)-R^\epsilon_\ast=1/2+\alpha$ while the adversarial surrogate risk converges to $\lim_{n\to \infty} R_\phi^\epsilon(f_n)=\phi(0)-C_\phi^\ast(1/2+\alpha)$. Consequently, 
    $$ \lim_{n\to \infty} \frac{R^\epsilon(f_n)-R^\epsilon_\ast}{R_\phi^\epsilon(f_n)-R^\epsilon_{\phi,\ast}} =\frac{\frac 12 +\alpha}{\phi(0)-C_\phi^\ast(1/2-\alpha)}.$$
    It follows that for any $\delta > 0$ there exists $f$ such that
    \[\cprm(f)-R_*^\e\geq \frac {1/2+\alpha}{\phi(0)-C_\phi^*(1/2-\alpha)}\left( \prm(f)-R_{\phi,*}^\e\right)-\delta.\]
    In particular, the constant in \Cref{thm:adv_surrogate_massart} is overestimated by factor of at most $1/( 1/2+\alpha)\leq 2$. However, this example demonstrates that \cref{thm:adv_surrogate_massart} is tight when $\alpha=1/2$. 
\end{example}
The constant in \cref{thm:adv_surrogate_massart} is known to be sub-optimal when $\alpha<1/2$. In particular, Theorem~4 of \cite{frank2024consistency} proves that $R^\e(f)- R^\e_*\leq (1/2)/(\phi(0)-C_{\phi_\rho}^*(1/2))(R_{\phi_\rho}^\e(f)-R^\e_{\phi_{\rho},*})$ for the $\rho$-margin loss $\phi_\rho(\alpha)=\min(1,\max(0,1-\alpha/\rho))$. We conjecture that the tight constant in \cref{thm:adv_surrogate_massart} is in fact $(1/2+\alpha)/(\phi(0)-C_\phi^*(1/2-\alpha))$. Together, these observations indicate that the bound in \cref{thm:adv_surrogate_massart} captures the correct order of dependence on $\alpha$ and $\phi$, and that only the numerican constant can potentially be improved.

%% file: Sections/5-convex_consistent_losses.tex
\section{Proof of \cref{thm:convex_surrogate_bound}}\label{sec:convex_surrogate_bound}

Before proving \cref{thm:convex_surrogate_bound}, we will show that this bound is non-vacuous when the adversarial Bayes classifier is unique up to degeneracy. The function $h(z)=\PP(|\eta^*-1/2|\leq z)$ is a cdf, and is thus right-continuous in $z$. Furthermore, if the adversarial Bayes classifier is unique up to degeneracy, then $h(0)=0$. The following lemma implies that if $H=\conc(h)$ then $H$ is continuous at 0 with $H(0)=0$. See \cref{app:envelope} for a proof. This result implies that the bound in \cref{thm:convex_surrogate_bound} is non-vacuous.
\begin{lemma}\label{lemma:envelope}
    Let $h:[0,1/2]\to \Rset$ be a non-decreasing function with $h(0)=0$ and $h(1/2)=1$ that is right-continuous at $0$. Then $\conc(h)$ is non-decreasing, continuous on $[0,1/2]$, and $\conc(h)(0)=0$. 
\end{lemma}

The first step in proving \cref{thm:convex_surrogate_bound} is showing an analog of \cref{thm:adv_surrogate_massart} with $\alpha=0$ for which the linear function is replaced by an $\eta$-dependent concave function. 

\begin{proposition}\label{prop:main_surrogate_bound}
    Let $\Phi$ be a concave non-decreasing function for which $C(\eta,\alpha)-C^*(\eta)\leq \Phi(C_\phi(\eta,\alpha)-C_\phi^*(\eta))$ for any $\eta\in [0,1]$ and $\alpha \in \ov \Rset$. Let $\PP_0^*$, $\PP_1^*$ be any two maximzers of $\dl$ for which $\PP^*(\eta^*=1/2)=0$, where $\PP^*=\PP_0^*+\PP_1^*$ and $\eta^*=d\PP_1^*/d\PP^*$. Let $G:[0,\infty)\to \Rset$ be any non-decreasing concave function for which the quantity
    \[K=\int \frac 1 {G( \phi(0)-C_\phi^*(\eta^*))}d\PP^*\]
    is finite. Then $\cprm(f) -R_{*}^\e\leq\td \Phi( \prm(f)-R_{\phi,*}^\e)$, where 
    \begin{equation}\label{eq:td_Phi}
        \td \Phi(z)= 4\sqrt{K G\left(\frac 14 z\right)}+2 \Phi\left(\frac 12 z\right)    
    \end{equation}
    
\end{proposition}
The proof strategy mirrors that of \cref{thm:adv_surrogate_massart}, but with $\Phi$ and $G$ replacing the fixed constant bound.

Uniqueness up to degeneracy and \cref{th:consistent_characterization} guarantee that the denominator $\phi(0)-C_\phi^*(\eta^*)$ is strictly positive $\PP^*$-a.e. The function $\Psi^{-1}$ in \Cref{prop:psi_def} and the surrogate bounds of \citet{zhang04} provide examples of candidate functions for $\Phi$.
As before, this result is proved by dividing the risks $\prm$, $\cprm$ 
as the sum of four terms as in \cref{eq:division_classification}, \cref{eq:division_surrogate} and then bounding these quantities over the sets $D_k$, $E_k$ defined in \cref{eq:D_1_orig},\cref{eq:E_1_orig} separately.

    The factor of $1/4$ in \cref{eq:td_Phi} arises as an artifact of the proof technique: one factor of 2 from averaging over two sets $D_1$, $E_1$, (see \cref{eq:G_contribution} in \cref{app:main_surrogate_bound}), and another factor of 2 from combining the bounds associated with the two integrals  corresponding to class 0 and class 1(see \cref{eq:opt_attack_risk_bound,eq:G_contribution} in \cref{app:main_surrogate_bound}).

We now turn to the problem of identifying functions $G$ for which the constant $K$ in the preceding proposition is guaranteed to be finite when the adversarial Bayes classifier is unique, but distribution dependent. Observe that if $h$ is the cdf of $|\eta-1/2|$ and $h$ is continuous, then $\int 1/h^r dh$ is always finite for $r\in(0,1)$. This calculation suggests $\Phi=h\circ \Psi^{-1}$, with $\Psi$ defined in \cref{prop:psi_def}. To ensure the concavity of $G$, we instead select $G=H\circ \Psi^{-1}$ with $H=\conc(h)$.

\begin{lemma}\label{lemma:multiple_rs_bound}
    Assume $C_\phi^*(1/2)=\phi(0)$. Let $\PP_1$, $\PP_0$, $\PP_1^*$, $\PP_0^*$, $\phi$, $H$, and $\Psi$ be as in \cref{thm:convex_surrogate_bound}. Let $\Lambda(z)=  \Psi^{-1}(\min(z,\phi(0)))$. Then for any $r\in (0,1)$, 
    \begin{equation}\label{eq:tilde_Phi_def}
        \cprm(f)-R^\e_*\leq \tilde \Phi(\prm(f) -R_{\phi,*}^\e)
    \end{equation}
     
    with 
    \begin{equation}\label{eq:proto_bound}
        \tilde \Phi (z)= 4\sqrt{\frac {1}{1-r} H\left( \frac 12 \Lambda\left(\frac 14 z\right)\right)^r} +2\Lambda\left(\frac z2 \right).    
    \end{equation}
    
\end{lemma}

\begin{proof}
    For convenience, let $G= (H\circ \frac 12\Lambda)^r$. Then $G$ is concave because it is the composition of a concave function and an increasing concave function. We will verify that $K$ is finite and yields the constant in the bound:
    \[K=\int \frac 1 {G( \phi(0)-C_\phi^*(\eta^*))}d\PP^* \leq \frac {1} {1-r}\]
    First,
    \begin{align}
        &\int \frac 1 {G( \phi(0)-C_\phi^*(\eta^*))}d\PP^*= \int \frac 1 {H(|\eta^*-1/2|)^r}d \PP^*=\int_{[0,\frac12]} \frac 1 {H(s)^r}d \PP^*\sharp s
        =\int_{(0,\frac12]} \frac 1 {H(s)^r}d \PP^*\sharp s\nonumber
    \end{align}
    with $s=|\eta^*-1/2|$. The assumption $\PP^*(|\eta^*=1/2|)=0$ allows us to drop 0 from the domain of integration. Because the function $H$ is continuous on $(0,1]$ by \cref{lemma:envelope}, this last expression can be evaluated as a Riemann-Stieltjes integral with respect to the function $h(s)=\PP(|\eta^*-1/2|\leq s)$:

    \begin{align}
        &\int_{(0,\frac 12]} \frac 1 {H(s)^r}d \PP^*\sharp s= \int_0^{1/2} \frac 1 {H(s)^r}dh\label{eq:lebesgue_to_rs}
    \end{align}
    This result is standard when $\PP^*$ is Lebesgue measure, (see for instance Theorem~5.46 of \citep{WheedenZygmund}). We prove equality in \cref{eq:lebesgue_to_rs} for strictly decreasing functions in \cref{prop:ls_to_rs_decreasing} in \cref{app:lebesge_RS}.
    
     Finally, the integral in \cref{eq:lebesgue_to_rs} can be bounded as 
    \begin{align}
        \int \frac 1 {H(s)^r}dh
        \leq \frac {1} {1-r}    \label{eq:1-r_int}
    \end{align}

     If $h$ were differentiable, then the chain rule would imply
     \[\int \frac 1 {H(s)^r}d h\leq \int \frac 1 {h(s)^r} dh=\int_0^{\frac 12} \frac 1 {h(s)^r} h'(s) dz= \frac 1 {1-r}h(s)^{1-r}\bigg|_0^{\frac 12}\leq \frac 1 {1-r}.\]
     
      This calculation is more delicate for non-differentiable $H$; we formally prove inequality in \cref{eq:1-r_int} in \cref{app:1-r_int}.
     
    This calculation proves the inequality \cref{eq:tilde_Phi_def} with $\tilde \Phi$ as \cref{eq:proto_bound}
\end{proof}

To obtain the bound in \cref{thm:convex_surrogate_bound}, first observe that     the concavity of $\Lambda$ together with the fact that $\Lambda(0)=0$ implies that $2\Lambda(z/2)\leq 4\Lambda(z/4)$.
Next, minimizing the bound in \cref{lemma:multiple_rs_bound} over $r$ then produces \cref{thm:convex_surrogate_bound}, see \cref{app:optimizing_bound} for details.

%% file: Sections/6-non-consistent_bound.tex
\section{Proof of \cref{thm:non_consistent_bound_linear,thm:non_consistent_bound_nonlinear}}\label{sec:non_consistent}
The main insight behind \cref{thm:non_consistent_bound_linear,thm:non_consistent_bound_nonlinear} is that a transport map that realizes the optimal adversarial perturbations also preserves optimality when restricted to certain subsets of $\Rset^d$, allowing a reduction from the global to a local problem in both the dual and primal formulations.
The following lemma formalizes the fact that under a transport map structure, restricting the primal problem to the pre-image of a set 
$Q$ corresponds exactly to restricting the dual maximizers to $Q$ itself.
\begin{lemma}\label{lemma:restrict_dual}
    Let $\PP_0,\PP_1$ be a data distribution and let $\PP_0^*\in \Wball \e(\PP_0),\PP_1^*\in \Wball\e(\PP_1)$ maximize $\dl$. Assume there exists transport maps $T_0,T_1$ for which $\PP_i^*=\PP_i\sharp T_i$ with $\|T_i(\bx)-\bx\|\leq \e$. Let $Q$ be any set and define $U_i=T_i^{-1}(Q)$. 

    If the data is distributed according to $\PP_0|_{U_0}, \PP_1|_{U_1}$, then $\PP_0^*|_{Q}$, $\PP_1^*|_{Q}$ maximize $\dl$ over $\Wball \e (\PP_0|_{U_0})\times \Wball \e(\PP_1|_{U_1})$.
\end{lemma}
In the remainder of this section it will be useful to include the data distribution in the notation for the primal problem. Thus, for the remainder of this section, we define
\begin{equation}\label{eq:augmented_primal}
    \prm(f;\PP_0,\PP_1)=\int S_\e(\phi\circ f) d\PP_1+\int S_\e(\phi\circ -f)d\PP_0\quad \cprm(f;\PP_0,\PP_1)=\int S_\e(\one_{f\leq 0})d\PP_1+\int S_\e(\one_{f>0})d\PP_0    
\end{equation}
Similarly, we'll denote 
\begin{equation}\label{eq:augmented_optimal}
    R^\e_{\phi,*}(\PP_0,\PP_1)= \inf_f  \prm(f;\PP_0,\PP_1),\quad R^\e_{*}(\PP_0,\PP_1)= \inf_f  \cprm(f;\PP_0,\PP_1)
\end{equation}

Observe that for any two sets $U_0$, $U_1$,
\[\prm(f;\PP_0,\PP_1)=\prm(f;\PP_0|_{U_0},\PP_1|_{U_1})+\prm(f;\PP_0|_{U_0^C},\PP_1|_{U_1^C})\]
This decomposition reflects the fact that the adversarial surrogate risk is additive over disjoint measurable partitions of the data space.
If furthermore these sets are induced by transport maps, then the optimal risks also follow this split.
\begin{lemma}\label{lemma:optimal_dual_split}
    Let $\PP_0^*,\PP_1^*,T_0,T_1,U_0,U_1$ and $Q$ be as in \cref{lemma:restrict_dual}, and define $\PP^*=\PP_0^*+\PP_1^*$, $\eta^*=d\PP_1^*/d\PP^*$. Then
    \[R^\e_{\phi,*}(\PP_0,\PP_1)=R^\e_{\phi,*}(\PP_0|_{U_0},\PP_1|_{U_1})+R^\e_{\phi,*}(\PP_0|_{U_0^C},\PP_1|_{U_1^C})\]
    and furthermore, $R^\e_{\phi,*}(\PP_0|_{U_0^C},\PP_1|_{U_1^C})=\int_{Q^C} C_\phi^*(\eta^*)d\PP^*$.
\end{lemma}
See \cref{app:non_consisten_lemmas} for a proof of \cref{lemma:restrict_dual,lemma:optimal_dual_split}.
\begin{proof}[Proof of \cref{thm:non_consistent_bound_nonlinear}]
    Let $Q=\{\bx': \eta^*(\bx')=1/2\}$. Then \cref{lemma:restrict_dual} applied to $Q^C$ shows that $(\PP_0^*|_{Q^C},\PP_1^*|_{Q^C})$ maximize $\dl$ over $\Wball \e (\PP_0|_{U_0})\times \Wball \e(\PP_1|_{U_1})$, with $U_0=T_0^{-1}(Q^C)$ and $U_1=T_1^{-1}(Q^C)$. \Cref{thm:convex_surrogate_bound,lemma:optimal_dual_split} imply that
    \begin{align*}
        \cprm(f;\PP_0|_{U_0},\PP_1|_{U_1})-R_*^\e(\PP_0|_{U_0},\PP_1|_{U_1})&\leq \tilde\Phi\left(\prm(f;\PP_0|_{U_0},\PP_1|_{U_1})-R_{\phi*}^\e(\PP_0|_{U_0},\PP_1|_{U_1})\right)\\
        &\leq  \tilde\Phi\left(\prm(f;\PP_0,\PP_1)-R_{\phi*}^\e(\PP_0,\PP_1)\right).
    \end{align*}
    Next, by \cref{lemma:optimal_dual_split}, adding $\cprm(f;\PP_0|_{U_0^C},\PP_1|_{U_1^C})- R_*^\e(\PP_0|_{U_0^C},\PP_1|_{U_1^C})$ to both sides of the inequality above results in 
    \[        \cprm(f;\PP_0,\PP_1)-R_*^\e(\PP_0,\PP_1)\leq  \tilde\Phi\left(\prm(f;\PP_0,\PP_1)-R_{\phi*}^\e(\PP_0|_{U_0},\PP_1)\right)+\cprm(f;\PP_0|_{U_0^C},\PP_1|_{U_1^C})- \int_QC^*(\eta^*)d\PP^*.\]
    The fact that $C^*(\eta^*)=1/2$ on $Q$ while $S_\e(\one_{f\leq0})\leq 1$, $S_\e(\one_{f>0})\leq 1$ implies that $\cprm(f;\PP_0|_{U_0^C},\PP_1|_{U_1^C})- \int_QC^*(\eta^*)d\PP^*\leq \frac 1 2\PP(\eta^*=1/2)$. Thus, the excess risk contribution from the region $Q$ is at most $\PP^*(\eta^*=1/2)/2$.
\end{proof}

The proof of \cref{thm:non_consistent_bound_linear} follows the same steps, except that we take $Q=\{\bx': |\eta(\bx')-1/2|<\alpha\}$, see \cref{app:non_consistent_bound_linear} for a proof.

%% file: Sections/7-Related_Works.tex
\section{Related Works}\label{sec:related_works}
\paragraph{Surrogate Risk Bounds:} The statistical consistency of surrogate risks in both the standard and adversarial context has been widely studied. \Citet{BartlettJordanMcAuliffe2006,zhang04} establish surrogate risk bounds that apply to the class of all measurable functions while  \citet{Lin2004,Steinwart2007} prove further results on consistency in the standard setting. \Citet{FrongilloWaggoner2021Polyhedral} study the optimally of such bounds, and \citep{Bao23moduli} derive bounds using the modulus of convexity of $C_\phi^*$ to construct surrogate risk bounds. Several works \Citep{LongServedioH-consistency,ZhangAgarwal,AwasthiMaoMohriZhong22Hconsistencybinary,MaoMohriZhong23crossentropy,MaoMohriZhong23structured,AwasthiMaoMohriZhong23grounded} study consistency within a restricted function class; a concept known as \emph{$\cH$-consistency}. \Citet{mahdavi2014smooth} combine surrogate risk bounds with surrogate generalization bounds to study the generalization of the classification error. 

\paragraph{Adversarial Surrogate Risk Bounds:} Most closely related to our results are \citep{LiTelgarsky2023achieving, MaoMohriZhong23crossentropy}. \citet{LiTelgarsky2023achieving} derive a surrogate bound for convex losses in which the threshold in \cref{eq:adv_classification_risk_alt} is optimized rather than fixed at zero. \citet{MaoMohriZhong23crossentropy} establish an adversarial surrogate bound for a modified $\rho$-margin loss.

\paragraph{Adversarial Consistency:} In the adversarial setting, \citep{MeunierEttedguietal22,FrankNilesWeed23consistency} characterize which losses are adversarially consistent for all data distributions. \Citet{frank2024consistency} show that under reasonable distributional assumptions, a consistent loss is adversarially consistent for a specific distribution iff the adversarial Bayes classifier is unique up to degeneracy. \Citep{AwasthiFrankMao2021} study adversarial consistency for a well-motivated class of linear functions while \citet{AwasthiMaoMohriZhong23grounded, MaoMohriZhong23crossentropy} study $\cH$-consistency in the adversarial setting for specific surrogate risks. Standard and adversarial surrogate risk bounds are a central tool in the derivation of the $\cH$-consistency bounds in this line of research.  Whether the adversarial surrogate bounds presented in this paper could result in improved adversarial $\cH$-consistency bounds remains an open problem.

\paragraph{The Adversarial Bayes Classifier:} Our proofs draw on prior work that investigates adversarial risks and adversarial Bayes classifiers.
\citet{BungertGarciaMurray2021,PydiJog2021,PydiJog2020,BhagojiCullinaMittalji2019lower,AwasthiFrankMohri2021} establish existence results for the adversarial Bayes classifier, while
\citet{FrankNilesWeed23minimax,PydiJog2020,PydiJog2021,BhagojiCullinaMittalji2019lower,frank2024consistency} prove minimax theorems for adversarial surrogate and classification risks.
\citet{PydiJog2020} use such results to analyze the adversarial Bayes classifier, and \citet{frank2024uniquness} employ them to study uniqueness.

\paragraph{Sample Complexity and Surrogate Risks:}The bound of \cref{prop:psi_def} can be linear even for convex loss functions. For the hinge loss $\phi(\alpha)=\max(1-\alpha,0)$, the function $\phi$ computes to $\phi(\theta)=|\theta|$. 
\Citet{mahdavi2014smooth} emphasize the importance of a linear convergence rate in a surrogate risk bound. They note that convex surrogates with favorable sample complexity often fail to satisfy strong surrogate risk bounds, due to Theorem~2 \citet{FrongilloWaggoner2021Polyhedral}: convex losses which are locally strictly convex and Lipschitz achieve at best a square root surrogate risk rate.
Thus, \cref{prop:surrogate_massart} suggests that favorable sample complexity guarantees for convex surrogates may require distributional conditions such as Massart’s noise condition, under which \citet{Massart06} also show improved sample complexity.

%% file: Sections/8-conclusion.tex
\section{Conclusion}
In conclusion, we prove surrogate risk bounds for adversarial risks. When $\phi$ is adversarially consistent or the distribution of optimal adversarial attacks satisfies Massart's noise condition, we obtain a linear surrogate risk bound. In the general case, we prove a concave distribution-dependent bound. Understanding the optimality of the concave bound remains an open problem, as does understanding how these bounds interact with the sample complexity of estimating the surrogate risk. While related questions have been studied in the standard setting \citep{FrongilloWaggoner2021Polyhedral,mahdavi2014smooth}, the adversarial context remains largely unexplored. Advancing these directions could bridge the current gap between theoretical guarantees and practical robustness in adversarial learning.

%% file: Sections/9-acks.tex
\subsubsection*{Acknowledgments}
Natalie Frank was supported in part by the Research Training Group in Modeling and Simulation
funded by the National Science Foundation via grant RTG/DMS – 1646339, NSF grant DMS-2210583, and NSF TRIPODS II - DMS 2023166.

%% file: appendix_list.tex
\input{Appendices/A-Consistent_charaterization}
\input{Appendices/B-non_adv_linear_surrogate_bound}
\input{Appendices/C-S_e_and_W_inf}
\input{Appendices/D-maximizer_comparison}

\input{Appendices/E-envelope}
\input{Appendices/F-General_surrogate}
\input{Appendices/G-1-r}
\input{Appendices/H_optimizing_bound}
\input{Appendices/I-non_consistent_deferred}
\input{Appendices/J-gaussian_and_massart}

%% file: Appendices/A-Consistent_charaterization.tex
\section{Proof of \cref{th:consistent_characterization}}\label{app:consistent_characterization}
\begin{lemma}\label{prop:consistent_at_1/2}
Assume $\phi$ is consistent. Then $C_\phi^*(\eta)=\phi(0)$ implies that $\eta=1/2$.
    
\end{lemma}
This result appeared as Lemma~7 of \citep{frank2024consistency}.
\begin{proof}
 If $\phi$ is consistent and $0$ minimizes $C_\phi(\eta,\alpha)$, then $0$ must also minimize $C(\eta,\alpha)=\eta\one_{\alpha\leq 0}+(1-\eta)\one_{\alpha>0}$ and consequently $\eta\leq 1/2$. However $C_\phi(\eta,\alpha)=C_\phi(1-\eta,-\alpha)$ so that $0$ must minimize $C(1-\eta,-\alpha)$ as well. Consequently, $1-\eta\leq 1/2$ and thus $\eta$ must actually equal $1/2$.
\end{proof}

\begin{proof}[Proof of \cref{th:consistent_characterization}]
    \textbf{Forward direction: } Assume that $\phi$ is consistent. Note that $C_\phi^*(\eta)\leq C_\phi(\eta,0)=\phi(0)$ for any $\eta$. Thus \cref{prop:consistent_at_1/2} implies that $C_\phi^*(\eta)<\phi(0)$ for $\eta\neq 1/2$.

    \textbf{Backward direction: } Assume that $C_\phi^*(\eta)<\phi(0)$ for all $\eta\neq 1/2$. Notice that if $\eta=1/2$, $C(1/2,\alpha)$ is constant in $\alpha$ so \emph{any} sequence $\alpha_n$ minimizes $C(1/2,\cdot)$. We will show if $\eta>1/2$ and $\alpha_n$ is a minimizing sequence of $C_\phi(\eta, \cdot)$, then $\alpha_n>0$ for sufficiently large $n$, and thus must also minimize $C(\eta,\cdot)$. An analogous argument will imply that if $\eta<1/2$, any minimizing sequence of $C_\phi(\eta,\cdot)$ must also minimize $C(\eta,\cdot)$ as well.

    Assume $\eta>1/2$ and let $\alpha_n$ be any minimizing sequence of $C_\phi(\eta,\cdot)$. Let $\alpha^*$ be a limit point of the sequence $\alpha_n$ in the extended real number line $\ov \Rset$. Then $\alpha^*$ is a minimizer of $C_\phi(\eta,\alpha)$. Next, observe that one of $\phi(\alpha^*)$, $\phi(-\alpha^*)$ is larger that or equal to $\phi(0)$ and the other is less than or equal to $\phi(0)$. As $\eta>1/2$ and $\alpha^*$ is a minimizer of $C_\phi(\eta,\cdot)$ and $C_\phi(\eta,\alpha^*)<\phi(0)$, one can conclude that $\phi(\alpha^*)<\phi(0)$ and consequently $\alpha^*>0$. 

    Therefore, every limit point of the sequence $\{\alpha_n\}$ is strictly positive. Consequently, one can conclude that $\alpha_n>0$ for sufficiently large $n$.

\end{proof}

%% file: Appendices/B-non_adv_linear_surrogate_bound.tex
\section{Linear Surrogate Risk Bounds---Proof of \cref{prop:surrogate_massart}}\label{app:surrogate_linear}\label{app:non_adv_surrogate_bounds}

In this appendix, we will find it useful to study the function 
    \[C_\phi^-(\eta)=\inf_{z(2\eta-1)\leq 0} C_\phi(\eta,z)\]  introduced by \citep{BartlettJordanMcAuliffe2006}. This function maps $\eta$ to the smallest value of the conditional $\phi$-risk assuming an incorrect classification.
    The symmetry $C_\phi(\eta,\alpha)=C_\phi(1-\eta,-\alpha)$ implies $C_\phi^-(\eta)=C_\phi^-(1-\eta)$. Further, the function $C_\phi^-$ is concave on each of the intervals $[0,1/2]$ and $[1/2,1]$, as it is an infimum of linear functions on each of these regions. The next result examines the monotonicity properties of $C_\phi^*$ and $C_\phi^-$.
\begin{lemma}\label{lemma:C_phi^*_monotonicity}
    The function $C_\phi^*$ is non-decreasing on $[0,1/2]$ and non-increasing on $[1/2,1]$. In contrast, $C_\phi^-$ is non-increasing on $[0,1/2]$ and non-decreasing on $[1/2,1]$
\end{lemma}
\begin{proof}
    The symmetry $C_\phi^*(\eta)=C_\phi^*(1-\eta)$ and $C_\phi^{-}(\eta)=C_\phi^{-}(1-\eta)$ implies that it suffices to check monotonicity on $[0,1/2]$. Observe that 
    \[C_\phi(\eta,\alpha)-C_\phi(\eta,-\alpha)= \eta(\phi(\alpha)-\phi(-\alpha))+(1-\eta)(\phi(-\alpha)-\phi(\alpha))=(2\eta-1)(\phi(\alpha)-\phi(-\alpha)).\]
    If $\eta\leq 1/2$, then this quantity is non-negative when $\alpha\leq 0$.
    Therefore, when computing $C_\phi^*$ over $[0,1/2]$, it suffices to minimize $C_\phi(\eta,\alpha)$ over $\alpha\leq 0$. In other words, for $\eta\leq 1/2$,
    \begin{equation*}
        C_\phi^*(\eta)=\inf_\alpha C_\phi(\eta,\alpha)=\inf_{\alpha\leq 0} C_\phi(\eta,\alpha)    
    \end{equation*}
    For any fixed $\alpha \leq 0$, the quantity $C_\phi(\eta,\alpha)$ is non-increasing in $\eta$ and thus $C_\phi^*(\eta_1)\leq C_\phi^*(\eta_2)$ when $\eta_1\leq \eta_2\leq 1/2$.

    In contrast, for any $\alpha\geq 0$, the quantity $C_\phi(\eta,\alpha)$ is non-decreasing in $\eta$ and thus $C_\phi^-(\eta_1)\geq C_\phi^-(\eta_2)$ when $\eta_1\leq \eta_2\leq 1/2$.

\end{proof}

Next we'll prove a useful lower bound on $C_\phi^-$.

\begin{lemma}\label{lemma:missclassifying_condition}
    For all $\eta\in [0,1]$,
    \begin{equation}\label{eq:C^-_lower_bound}
        C_\phi^-(\eta)\geq |1-2\eta|\phi(0)+2\min(\eta,1-\eta) C_\phi^*(\eta)
    \end{equation}
    
\end{lemma}

\begin{proof}
    First, assume that $\eta\leq 1/2$ and observe that $\eta$ is the convex combination $\eta=2\eta\cdot 1/2+(1-2\eta)\cdot 0$. By the concavity of $C_\phi^-$ on $[0,1/2]$,
    \begin{align*}
        C_\phi^-(\eta)&=C_\phi^-\left (2\eta\cdot \frac 12+(1-2\eta)\cdot 0\right )\geq (1-2\eta)C_\phi^-(0)+2\eta C_\phi^-\left (\tfrac 12 \right)
    \end{align*}

    However, $C_\phi^-(0)=\phi(0)$ while $C_\phi^-(1/2)=C_\phi^*(1/2)$. Further, \cref{lemma:C_phi^*_monotonicity} implies that $C_\phi^*(1/2)\geq C_\phi^*(\eta)$, yielding the inequality
    \[C_\phi^-(\eta)\geq (1-2\eta)\phi(0)+2\eta C_\phi^*(\eta)\]
    Symmetry $C_\phi^-(\eta)=C_\phi^-(1-\eta)$ then implies \cref{eq:C^-_lower_bound}.
    
\end{proof}

\begin{proof}[Proof of \cref{prop:surrogate_massart}]
If $C(\eta, f)-C^*(\eta)=0$ then  \cref{eq:massart_loss} holds trivially. Otherwise, $C(\eta, f)-C^*(\eta)=|2\eta-1|$.
If $C(\eta,f)=|2\eta-1|$, then   
\begin{equation}\label{eq:massart_initial_bound}
\begin{aligned}
     C(\eta, f)-C^*(\eta)&= |2\eta-1|=|2\eta-1|\cdot \frac{\phi(0)-C_\phi^*(\eta)}{\phi(0)-C_\phi^*(\eta)}\\ &\leq \frac 1 {\phi(0)-C_\phi^*(\eta)}\left(\left(|2\eta-1|\phi(0)+(1-|2\eta-1|)C_\phi^*(\eta)\right) -C_\phi^*(\eta)\right)
\end{aligned}    
\end{equation}

At the same time,
because $|\eta-1/2|\geq \alpha$ $\PP$-a.e. \cref{lemma:C_phi^*_monotonicity} implies that $C_\phi^*(\eta)\leq C_\phi^*(1/2-\alpha)$ $\PP$-a.e. Furthermore, the relation $2\min(\eta,1-\eta)=1-|1-2\eta|$ together with \cref{eq:C^-_lower_bound} shows that \[|2\eta-1|\phi(0)+(1-|2\eta-1|)C_\phi^*(\eta)\leq C_\phi^-(\eta).\] Therefore, \cref{eq:massart_initial_bound} is bounded above by
\begin{align}
&\leq \frac 1 {\phi(0)-C_\phi^*\left(\frac 12 -\alpha\right)}\left( C_\phi^-(\eta)-C_\phi^*(\eta) \right)\leq \frac 1 {\phi(0)-C_\phi^*\left(\frac 12 -\alpha\right)} \left( C_\phi(\eta,f)-C_\phi^*(\eta)\right).\label{eq:massart_loss_last}
\end{align}
The last equality follows from the supposition $C(\eta,f)-C^*(\eta)=|2\eta-1|$, as it implies $(2\eta-1)f\leq 0$, and thus $C_\phi(\eta,f)\geq C_\phi^-(\eta)$. Consequently, \cref{eq:massart_loss_last} implies \cref{eq:massart_loss}.

Integrating \cref{eq:massart_loss} with respect to $\PP$ then produces the surrogate bound \cref{eq:massart_surrogate}.

\end{proof}

%% file: Appendices/C-S_e_and_W_inf.tex
\section{Proof of \cref{lemma:S_e_and_W_inf}}\label{app:S_e_and_W_inf}

\begin{proof}[Proof of \cref{lemma:S_e_and_W_inf}]
    If $\bx'\in \ov{B_\e(\bx)}$ then $S_\e(g)(\bx)\geq g(\bx')$. Thus if $\gamma$ is a coupling between $\QQ$ and $\QQ'$ supported on $\Delta_\e$, then $S_\e(g)(\bx)\geq g(\bx')$ $\gamma$-a.e. Integrating this inequality in $\gamma$ produces
    \[\int S_\e(g)d\QQ\geq \int gd\QQ'.\]
    Taking the supreumum over all $\QQ\in \Wball \e(\QQ)$ then proves the result.
\end{proof}

%% file: Appendices/D-maximizer_comparison.tex
\section{Proof of \cref{it:maximizers_match_1}, \cref{th:maximizer_comparison}}\label{app:maximizer_comparison}

We will work with an alternative primal problem from \citep{FrankNilesWeed23minimax} that will make it easier to study the dual. Consider minimizing
\[\Theta(h_0,h_1)=\int S_\e(h_1)d\PP_1+\int S_\e(h_0) d\PP_0\] over the convex set
    \begin{equation}\label{eqn:S_def}
        S_\phi=\left\{ \begin{aligned}
        &(h_0,h_1)\colon h_0,h_1\colon K^\e \to \ov \Rset\text{ Borel, }0\leq h_0,h_1 \text{ and for}\\
        &  \text{all $\bx\in \Rset^d$,  $\eta\in[0,1]$, } \eta h_1(\bx)+(1-\eta)h_0(\bx)\geq C_\phi^*(\eta)
        \end{aligned}\right\}
    \end{equation}

Then strong duality holds with $\Theta$ in place of $\prm$. Furthermore, there exist minimizers over the set of $\ov \Rset$-valued functions, where $\ov \Rset=\Rset\cup\{-\infty,+\infty\}$.

\begin{theorem}\label{thm:theta_strong_duality}
    Define $\dl$ as in \cref{eq:dl_def}.
    \[\inf_{(h_0,h_1)\in S_\phi} \Theta(h_0,h_1)=\sup_{\substack{\PP_0'\in \Wball \e(\PP_0)\\ \PP_1' \in \Wball \e(\PP_1)}} \dl(\PP_0',\PP_1')\]
    Furthermore, the infimum is attained at some $\ov \Rset$-valued $h_0^*$, $h_1^*$.
\end{theorem}
See \citep[Lemma~14,Lemma~21]{FrankNilesWeed23minimax} for a proof of this result. \Cref{th:strong_duality_surrogate} already implies that the dual problem attains its supremum. Complimentary slackness conditions further characterize minimizers and maximizers.

\begin{theorem}[Complementary Slackness]\label{thm:comlementary_slackness_theta}
The pair $(h_0^*,h_1^*)$ minimize
$\Theta$ over $S_\phi$ and the measures $(\PP_0^*, 
\PP_1^*)$ maximize $\dl$ over 
$\Wball \e(\PP_0)\times \Wball \e(\PP_1)$ 
iff the following two conditions hold:
\begin{enumerate}[label=\arabic*)]
    \item\label{it:sup_theta} \[\int S_\e(h_1^*)d\PP_1=\int h_1^*d\PP_1^*\quad \text{and}\quad \int S_\e(h_0^*)d\PP_0=\int h_0^*d\PP_0^*\]
    \item \label{it:eq_comp_slack_theta}\[\eta^*h_1^*+(1-\eta^*)h_0^*=C_\phi^*(\eta^*)\quad \PP^*\text{-a.e.}\]
\end{enumerate}
\end{theorem}
See \citep[Lemma~15]{FrankNilesWeed23minimax} for a proof. \Cref{thm:theta_strong_duality,thm:comlementary_slackness_theta} apply to the conditional risk $C^*(\eta)$ as $C^*(\eta)=C_\phi^*(\eta)$ for the hinge $\phi(\alpha)=\frac 12 (1-\alpha)_+$.

We will use a characterization of consistency similar to \cref{th:consistent_characterization} in the proof of \cref{it:maximizers_match_1}, \cref{th:maximizer_comparison}.

\begin{theorem}\label{th:consistent_characterization_2}
    A loss function $\phi$ is consistent iff $C_\phi^*(\eta)$ has a strict maximum at $1/2$.
\end{theorem}
\begin{proof}
    If $C_\phi^*(1/2)=\phi(0)$, this statement is exactly \cref{th:consistent_characterization}. If $C_\phi^*(1/2)<\phi(0)$, \Citep[Proposition~3]{FrankNilesWeed23consistency} implies that $\phi$ is consistent. It remains to show that if $C_\phi^*(1/2)<\phi(0)$, then $C_\phi^*(\eta)$ has a strict maximum at $1/2$. As every sequence has a convergent subsequence in $\ov \Rset$, one can assume that $C_\phi(1/2,\cdot)$ has a minimizer $\alpha^*$ and $C_\phi^*(1/2)<\phi(0)$ implies $\alpha^*\neq 0$. Symmetry of $C_\phi(1/2,\cdot)$ implies that we can assume $\alpha^*>0$, and thus $\phi(\alpha^*)\leq \phi(0)$ and $\phi(-\alpha^*)\geq \phi(0)$. The fact that $C_\phi^*(1/2,\alpha^*)<\phi(0)$ implies that in fact $\phi(\alpha^*)<\phi(0)\leq \phi(-\alpha^*)$. Next, observe that for any $\alpha$,
    \[C_\phi(\eta,\alpha)=\frac 12 (\phi(\alpha)+\phi(-\alpha))+(\eta-\frac 12) (\phi(\alpha)-\phi(-\alpha))\]
    Thus, one can bound $C_\phi^*(\eta)$ by 
    \[C_\phi^*(\eta)\leq C_\phi(\eta,\alpha^*)=\frac 12(\phi(\alpha^*)+\phi(-\alpha^*))+(\eta-1/2)(\phi(\alpha^*)-\phi(-\alpha^*))=C_\phi^*(1/2)+(\eta-1/2)(\phi(\alpha^*)-\phi(-\alpha^*))\]
    Thus if $\eta>1/2$, then $C_\phi^*(\eta)<C_\phi^*(1/2)$. Symmetry implies that $C_\phi^*(\eta)<C_\phi^*(1/2)$ for all $\eta$. Thus $C_\phi^*$ has a strict maximum at 1/2.

\end{proof}

Next, \cref{thm:comlementary_slackness_theta} implies that minimizers of $\Theta$ assume their suprema. This observation will make it easier to work with these functions.

\begin{lemma}\label{lemma:sup_assume}
    If $(h_0^*,h_1^*)$ minimizes $\Theta$ over $S_\phi$, then the functions $h_0^*$, $h_1^*$ assume their suprema $\PP_0$-a.e. and $\PP_1$-a.e. respectively
\end{lemma}
\begin{proof}
    We will show the statement for $h_1^*$, the argument for $h_0^*$ is analogous. Let $\gamma_1^*$ be the coupling between $\PP_1$ and $\PP_1^*$ that achieves the minimum $W_\infty$ distance. \Cref{lemma:S_e_and_W_inf} and \cref{it:sup_theta} of \cref{thm:comlementary_slackness_theta} implies that 
    \[S_\e(h_1)(\bx)=h_1(\bx') \quad \gamma_1^*\text{-a.e.}\]
    and thus $h_1^*$ assumes its maximum over closed $\e$-balls $\PP_1^*$-a.e.
\end{proof}

\begin{lemma}\label{lemma:alternative_ineq}
    If $(h_0^*,h_1^*)\in S_\phi$, then at any $\bx$ either $h_1^*(\bx)\geq C_\phi^*(\frac 12)$ or $h_0^*(\bx)> C_\phi^*(\frac 12)$. 
\end{lemma}
\begin{proof}
    If $(h_0^*,h_1^*)\in S_\phi$, then at any point $\bx$,
     \[\frac 12 h_0^*(\bx)+\frac 12 h_1^*(\bx)\geq C_\phi^*(\frac 12).\]
    The inequality $h_0^*(\bx)\leq C_\phi^*(1/2)$ implies $h_1^*(\bx)\geq C_\phi^*(1/2)$. Thus either $h_0^*(\bx)>C_\phi^*(1/2)$ or $h_1^*(\bx)\geq C_\phi^*(1/2)$ at any point.
\end{proof}

\begin{proof}[Proof of \cref{it:maximizers_match_1} of \cref{th:maximizer_comparison} ]
 Let $\hingephi(\alpha)=\frac 12 (1-\alpha)_+$, then $C_\hingephi^*(\eta)=C^*(\eta)$. 
 
 Let $(h_0^*,h_1^*)$ minimize $\Theta$ over $S_\phi$ and $(\PP_0^*,\PP_1^*)$ maximize $\dl$ over $\Wball \e(\PP_0)\times \Wball \e(\PP_1)$. We will show that the functions defined by 
 \[\tilde h_1^*(\bx)=\one_{h_1^*(\bx)\geq C_\phi^*(\frac 12)} \quad  \tilde h_0^*(\bx)=\one_{h_0^*(\bx)> C_\phi^*(\frac 12)}\]
 maximize $\Theta$ over $S_\hingephi$ and $(\PP_0^*,\PP_1^*)$ maximize $\bar R_{\hingephi}$ by verifying the constraint $(\td h_0^*, \td h_1^*)\in S_{\hingephi}$ and the complimentary slackness conditions. The proof thus consists of three steps: verifying $(\td h_0^*,\td h_1^*)\in S_{\hingephi}$, and checking the two complementary slackness conditions in \cref{thm:comlementary_slackness_theta}.

 \begin{enumerate}[label=\arabic*)]
    \item \textbf{Verifying the constraint defining $S_{\hingephi}$:} 
            Observe that \cref{lemma:alternative_ineq} implies that at any $\bx$, at least one of $\tilde h_0^*(\bx)$ and $\td h_1^*(\bx)$ is 1, and thus 
            \[\eta h_1^*(\bx)+(1-\eta) h_0^*(\bx)\geq \min(\eta,1-\eta)= C_\hingephi^*(\eta)\]
    \item \textbf{Verifying \cref{it:sup_theta} of \cref{thm:comlementary_slackness_theta}:} Observe that \cref{lemma:sup_assume} implies that $S_\e(\one_{h_1^*\geq C_\phi^*(1/2)})(\bx)=\one_{S_\e(h_1^*)(\bx)\geq C_\phi^*(1/2)}$ $\PP_1^*$-a.e. Subsequently, the \cref{it:sup_theta} of \cref{thm:comlementary_slackness_theta} implies that
    \[S_\e(\one_{h_1^*\geq C_\phi^*(1/2)})(\bx)=\one_{h_1^*(\bx')\geq C_\phi^*(1/2)}\quad \gamma_1^*\text{-a.e.},\]
    verifying the first complimentary slackness condition for $\td h_1^*$. Analogous reasoning shows that 
        \[S_\e(\one_{h_0^*> C_\phi^*(1/2)})(\bx)=\one_{h_0^*(\bx')> C_\phi^*(1/2)}\quad \gamma_0^*\text{-a.e.}\]

    \item \textbf{Verifying \cref{it:eq_comp_slack_theta} of \cref{thm:comlementary_slackness_theta}:}  \Cref{thm:comlementary_slackness_theta} implies that $\eta^* h_1^*(\bx')+(1-\eta^*) h_0^*(\bx')=C_\phi^*(\eta^*)\leq C_\phi^*(1/2)$, and thus \cref{lemma:alternative_ineq} implies that \emph{exactly} one of $h_1^*(\bx')$ and $h_0^*(\bx')$ equals 1 and the other equals 0. We'll consider the cases $\eta^*(\bx')<1/2$, $\eta^*(\bx')=1/2$, and $\eta^*(\bx')>1/2$ separately. In these three separate cases, we will explicitly use the formula $C_{\hingephi}^*(\eta)=\min(\eta,1-\eta)$.
    \begin{description}[leftmargin=3em, labelindent=3em]
        \item[When $\eta^*(\bx')=1/2$:] 
        As exactly one of $h_0^*(\bx')$ and $h_1^*(\bx')$ is 1:  
        \[\frac 12 \td h_0^*(\bx')+\frac 12 \td h_1^*(\bx')=\frac 12 =C_{\hingephi}^*\left(\frac 12\right)\]
        \item[When $\eta^*(\bx')<1/2$:]
                Observe that if $h_0^*(\bx')>h_1^*(\bx')$, then 
        \[\eta^*h_0^*(\bx')+(1-\eta^*)h_1^*(\bx')<\eta^*h_1^*(\bx')+(1-\eta^*)h_0^*(\bx')=C_\phi^*(\eta^*),\]
        which would violate the constraint on $S_\phi$. Therefore, $h_0^*(\bx')\leq h_1^*(\bx')$. Next, \cref{th:consistent_characterization_2} implies that $\eta^* h_1^*(\bx')+(1-\eta^*) h_0^*(\bx')=C_\phi^*(\eta^*)< C_\phi^*(1/2)$. These two statements together with \cref{lemma:alternative_ineq} imply that $h_0^*(\bx')< C_\phi^*(1/2)$ and $h_1^*(\bx')\geq C_\phi^*(1/2)$. However, $h_0^*(\bx')= C_\phi^*(1/2)$ would still violate $\eta^* h_1^*(\bx')+(1-\eta^*) h_0^*(\bx')< C_\phi^*(1/2)$ and therefore, $h_0^*(\bx')<C_\phi^*(1/2)$. Therefore,
        \[\eta^*\tilde h_1^*+(1-\eta^*) \tilde h_0^*=\eta^*=C_{\hingephi}^*(\eta^*)\]
        \item[When $\eta^*(\bx')>1/2$:]  Argument is analogous to the previous case.
    \end{description}

 \end{enumerate}

\end{proof}

%% file: Appendices/E-envelope.tex
\section{Proof of \cref{lemma:envelope}}\label{app:envelope}

We define the \emph{concave conjugate} of a function $h$ as 
\[h_*(y)=\inf_{x\in \dom(h)} yx-h(x)\]
Recall that $\conc (h)$ as defined in \cref{eq:conc_def} is the biconjugate $h_{**}$. Consequently, $\conc(h)$ can be expressed as

    \begin{equation}\label{eq:conc_linear_inf}
        \conc(h)(x)=\inf\{ \ell(x): \ell \text{ linear, and }\ell\geq h \text{ on }\dom(h)\}    
    \end{equation}

\Cref{lemma:envelope} is a consequence of the properties of concave conjugates. 
\begin{lemma}\label{lemma:non_dec_conj}
    Let $h:[a,b]\to \Rset$ be a non-decreasing function. Then $\conc(h)$ is non-decreasing as well. 
\end{lemma}
\begin{proof}

    We will argue that if $h$ is non-decreasing, then it suffices to consider the infimum in \cref{eq:conc_linear_inf} over non-decreasing linear functions. 
    Observe that if $\ell$ is a decreasing linear function with $\ell(x)\geq h(x)$ then the constant function $\ell(b)$ satisfies 
    \[\ell(x)\geq \ell(b)\geq h(b)\geq h(x)\] for any $x\in[a,b]$. Therefore, 
    \[\conc(h)(x)=\inf\{ \ell(x): \ell \text{ linear, non-decreasing, and }\ell\geq h\}\]
\end{proof}

\begin{lemma}\label{lemma:left_cont_0_to_0}
    Let $h:[0,b]\to \Rset$ be a non-decreasing function that is right-continuous at zero with $h(0)=0$. Then $\sup_y h_*(y)=0$. Furthermore, there is a sequence $y_n$ with $y_n\to \infty$ and $\lim_{n\to \infty}h_*(y_n)=0$.
\end{lemma}
\begin{proof}
    First, notice that 
    \begin{equation}\label{eq:zero_upper_bound}
        h_*(y)=\inf_{x\in [0,b]} yx-h(x)\leq y\cdot 0-h(0)=0    
    \end{equation}
    for any $y\in \Rset$. It remains to show a sequence $y_n$ for which $\lim_{n\to \infty} h_*(y_n)=0$. 

    We will argue than any sequence $y_n$ with 
    \begin{equation}\label{eq:y_n_def}
        y_n> nh(b)\geq \sup_{x\in [1/n,b]} \frac{h(x)}x    
    \end{equation}
     satisfies this property.

    If $x\in [1/n,b]$ and $y_n$ satisfies \cref{eq:y_n_def} then
    \[xy_n-h(x)=x\left (y_n-\frac {h(x)}x \right) > 0\] and thus \cref{eq:zero_upper_bound} implies that 
    \[h_*(y_n)=\inf_{x\in [0,1/n)} xy_n-h(x)\]
    The monononicity of $h$ then implies that 
    \[h_*(y_n)\geq -h(1/n) \]
    and 
    \[\lim_{n\to \infty} h_*(y_n)\geq 0\] 
    
    because $h$ is right-continuous at zero.
    This relation together with \cref{eq:zero_upper_bound} implies the result. 
    
\end{proof}

\begin{proof}[Proof of \cref{lemma:envelope}]

  \Cref{lemma:non_dec_conj} implies that $\conc(h)$ is non-decreasing. Standard results in convex analysis imply that $\conc(h)$ is continuous on $(0,1/2)$ \citep[Lemma~3.1.1]{Hiriarty-Urrty2001convex} and upper semi-continuous on $[0,1/2]$ \citep[Theorem~1.3.5]{Hiriarty-Urrty2001convex}. Thus monotonicity implies that for all $x\in [0,1/2]$, $\conc(h)(x)\leq \conc(h)(1/2)$ and thus $\lim_{x\to 1/2} \conc(h)(x)\leq \conc(h)(1/2)$. We will show the opposite inequality, implying that $\conc h$ is continuous at $1/2$.

    First, as the constant function $h(1/2)$ is an upper bound on $h$, one can conclude that $ \conc(h)(1/2)=h(1/2)=1$.
  Next, recall that $\conc(h)$ can be expressed as an infimum of linear functions as in \cref{eq:conc_linear_inf}. If $\ell\geq h$, then $\ell(0)\geq 0$ and $\ell(1/2)\geq 1$. Therefore,
    \[\ell(\tfrac 12 -\delta)=\ell( (1-2\delta)\cdot \tfrac 12 + 2\delta \cdot 0)=(1-2\delta)\ell(\tfrac 12)+2\delta \ell(0)\geq 1-2\delta.\]
    Therefore, the representation \cref{eq:conc_linear_inf} implies that $\conc(h)(1/2-\delta)\geq 1-2\delta$. Taking $\delta\to 0$ proves that $\lim_{x\to 1/2} \conc(h)(x)\geq 1$.
  Thus, $\conc(h)$ is continuous at $1/2$, if viewed as a function on $[0,1/2]$.

 Next, \cref{lemma:left_cont_0_to_0} implies that $h_{**}(0)=0$:
\[h_{**}(0)=\inf_{y\in \Rset} -h_*(y)=-\sup_{y\in \Rset}h_*(y)=0.\]

Finally, it remains to show that $h_{**}$ is continuous at 0. The monotonicity of $h_{**}$ implies that $\lim_{y\to 0^+} h_{**}(y)=\inf_{y\in (0,1/2]}h_{**}(y)$ and consequently
\begin{align}
&\lim_{y\to 0^+}h_{**}(y)=\inf_{y\in(0,1/2]}\inf_{x\in\Rset} yx-h_*(x)=\inf_{x\in\Rset}\inf_{y\in(0,1/2]} yx-h_*(x)=\inf_{x\in \Rset}-h_*(x)+\begin{cases}
0&\text{if }x\geq 0\\
\frac x2 &\text{if }x<0
\end{cases}\nonumber\\
&=\min\left(\inf_{x\geq 0} -h_*(x), \inf_{x<0} \frac x2 -h_*(x)\right)\label{eq:total_min}   
\end{align}

 However, \cref{lemma:left_cont_0_to_0} implies that 
 \begin{equation}\label{eq:first_min_compute}
     \inf_{x\geq 0} -h_*(x)=\inf_{x\in \Rset} -h_*(x)=0
 \end{equation}

Notice that if $x\leq0$,
\begin{equation}\label{eq:second_min_compute}
    h_*(x)=\inf_{z\in[0,1/2]} xz -h(z)=\frac x2-h\left(\frac 12\right)=\frac x2-1    
\end{equation}

Consequently, \cref{eq:first_min_compute} and \cref{eq:second_min_compute} implies that \cref{eq:total_min} evaluates to 0.

\end{proof}

%% file: Appendices/F-General_surrogate.tex
\section{Proof of \cref{prop:main_surrogate_bound}}\label{app:main_surrogate_bound}

A modified version of Jensen’s inequality will be used at several points in the proof of \cref{prop:main_surrogate_bound}.
\begin{lemma}\label{lemma:modified_Jensen}
    Let $G$ be a concave function with $G(0)=0$ and let $\nu$ be a measure with $\nu(\Rset^d)\leq 1$. Then 
    \[\int G(f)d\nu\leq G\left (\int f d\nu\right)\]
\end{lemma}
\begin{proof} The inequality trivially holds if $\nu(\Rset^d)=0$, so we assume $\nu(\Rset^d)>0$.
    Jensen's inequality implies that 
    \[\int G(f) d\nu=\nu(\Rset^d)\left( \frac 1 {\nu(\Rset^d)}\int G(f)d\nu\right)\leq \nu(\Rset^d) G\left(\frac 1 {\nu(\Rset^d)}\int f d\nu \right).\]
    As $G(0)=0$, concavity implies that 
    \[\nu(\Rset^d) G\left(\frac 1 {\nu(\Rset^d)}\int f d\nu \right)=\nu(\Rset^d) G\left(\frac 1 {\nu(\Rset^d)}\int f d\nu \right)+(1-\nu(\Rset^d)G(0)\leq G\left (\int f d\nu \right)\]
\end{proof}

To facilitate the application of Jensen’s inequality, the proof will be carried out using integrated quantities. Let $\PP_0^*,\PP_1^*$ be any maximizers of $\dl$, which also maximize $\cdl$ by \cref{th:maximizer_comparison}. Set $\PP^*=\PP_0^*+\PP_1^*$, $\eta^*=d\PP_1^*/d\PP^*$. Define

\begin{proof}[Proof of \cref{prop:main_surrogate_bound}]Let $\gamma_0^*$, $\gamma_1^*$ be the couplings between $\PP_0$, $\PP_0^*$ and $\PP_1$, $\PP_1^*$ respectively that achieve the infimum in \cref{eq:W_inf_def}.
    Define $I_1(f)$, $I_0(f)$, $I_1^\phi(f)$, and $I_0^\phi(f)$ by
    \[I_1(f)=\int i_1(f) d\gamma_1^*,\quad I_1^\phi(f)=\int i_1^\phi (f)d\gamma_1^*,\quad I_0(f)=\int i_0(f) d\gamma_0^*,\quad I_0^\phi(f)=\int i_0^\phi(f) d\gamma_0^*.\]
        We will prove 
    \noindent\begin{minipage}{0.48\linewidth}
    \begin{equation}\label{eq:general_I_0_bound}
        I_0(f)\leq \frac 12 \tilde \Phi\big(2I_0^\phi(f)\big).
    \end{equation}
    \end{minipage}%
    \begin{minipage}{0.48\linewidth}
    \begin{equation}\label{eq:general_I_1_bound}
        I_1(f)\leq  \frac 12 \tilde \Phi \big(2I_1^\phi(f)\big)
    \end{equation}
    \end{minipage}

    The concavity of $\tilde \Phi$ then implies that 
    \[\cprm(f)-R_*^\e=I_1(f)+I_0(f)\leq \frac 12 \tilde \Phi\big(2I_1^\phi(f)\big)+\frac 12 \tilde \Phi\big(2I_0^\phi(f)\big)\leq  \tilde \Phi\Big(\frac 122I_1^\phi(f)+\frac 12 2I_0^\phi(f)\Big)=\tilde \Phi\big(\prm(f)-R_{\phi,*}^\e\big).\]
    We will prove \cref{eq:general_I_1_bound}, the argument for \cref{eq:general_I_0_bound} is analogous.
    Next, let $\gamma_1^*$ be the coupling between $\PP_1$ and $\PP_1^*$ supported on $\Delta_\e$.
    The assumption on $\Phi$ implies that 
    \begin{equation}\label{eq:adv_attack_bound_1_general}
        C(\eta^*(\bx'),f(\bx'))-C^*(\eta^*(\bx'))\leq \Phi\big(C_\phi(\eta^*(\bx'),f(\bx'))-C_\phi^*(\eta^*(\bx'))\big)
    \end{equation}
    and consequently, 
    \begin{equation}\label{eq:opt_attack_risk_bound}
        \int C(\eta^*(\bx'),f(\bx'))-C^*(\eta^*(\bx'))d\gamma_1^*\leq \Phi\left(\int C_\phi(\eta^*(\bx'),f(\bx'))-C_\phi^*(\eta^*(\bx'))d\gamma_1^*\right)\leq \Phi(I_1^\phi(f)).
    \end{equation}
    To bound the term $S_\e(\one_{f\leq 0})(\bx)-\one_{f(\bx')\leq 0}$, we consider two different cases for $(\bx,\bx')$. Define the sets $D_1$, $E_1$ as in \cref{eq:D_1_orig}, \cref{eq:E_1_orig}.
    We will show that if $T_1$ is any of the sets $D_1$, $E_1$, then 
    \begin{equation}\label{eq:conv_goal}
    \begin{aligned}
        &\int_{T_1} S_\e(\one_{f\leq 0})(\bx)-\one_{f(\bx')\leq 0}d\gamma_1^*\\
        &\leq  \left(\int\frac {1}{ G\Big(\big(\phi(0)-C_\phi^*(\eta^*(\bx'))\big)/2\Big)} d\gamma_1^*\right)^{\frac 12}G\left(\int_{T_1} \big((S_\e(\phi\circ f)(\bx)-\phi(f(\bx'))\big)+\left(C_\phi(\eta^*(\bx'),f(\bx'))-C_\phi^*(\eta^*(\bx'))\right)d\gamma_1^*\right)^{\frac 12}     
    \end{aligned}
    \end{equation}

    Thus because $G$ is concave and non-decreasing, the composition $\sqrt{G}$ is as well. Thus summing the inequality \cref{eq:conv_goal} over $T_1\in \{D_1, E_1\}$ results in 
    \begin{align}
    \int S_\e(\one_{f\leq 0})(\bx)-\one_{f(\bx')\leq 0}d\gamma_1^* &  \leq 2\left(\int\frac {1}{ G\Big(\phi(0)-C_\phi^*(\eta^*(\bx'))\Big)} d\PP^*\right)^{\frac 12} G\left( \frac 12I_1^\phi(f) \right)^{\frac 12}\label{eq:G_contribution}
    \end{align}

    Summing \cref{eq:opt_attack_risk_bound} and \cref{eq:G_contribution} results in \cref{eq:general_I_1_bound}.

    It remains to show the inequality \cref{eq:conv_goal} for the two sets $D_1$, $E_1$.
    \begin{enumerate}[label=\Alph*)]
        \item \textbf{On the set $D_1$:}\\
        
        If $S_\e(\one_{f\leq 0})(\bx)=\one_{f(\bx')\leq 0}$, then $\int_{D_1} S_\e(\one_{f\leq 0})(\bx)-\one_{f(\bx')\leq 0}d\gamma_1^*=0$ while the left-hand side of \cref{eq:conv_goal} is non-negative by \cref{lemma:S_e_and_W_inf}, which implies \cref{eq:conv_goal} for $T_1=D_1$.

    \item \textbf{On the set $E_1$:}

    \Cref{lemma:S_e_and_W_inf} then implies that $S_\e(\phi \circ f)(\bx)-\phi(f(\bx'))\geq 0$ $\gamma_1^*$-a.e. and thus \cref{lemma:DEF_surrogate_lower_bound} implies 
    \begin{equation}\label{eq:general_1_comparison}
        S_\e(\one_{f\leq 0})(\bx)-\one_{f(\bx')\leq 0}= 1=\frac{\sqrt{G\left(\phi(0)-C_\phi^*(\eta^*(\bx'))\right)}}{ \sqrt{G\left(\phi(0)-C_\phi^*(\eta^*(\bx'))\right)}}\leq  \frac {\sqrt{G\big(S_\e(\phi\circ f)(\bx)-\phi(f(\bx'))\big)}}{\sqrt{G\left(\phi(0)-C_\phi^*(\eta^*(\bx'))\right)}}\quad \gamma_1^*\text{-a.e.}
    \end{equation}
    Now the Cauchy-Schwartz inequality and Jensen's inequality(\cref{lemma:modified_Jensen}) imply 
    \begin{equation}\label{eq:general_reasoning_transfer_2}
    \begin{aligned}
        &\int_{E_1} S_\e(\one_{f\leq 0})(\bx)-\one_{f(\bx')\leq 0}d\gamma_1^*\\&\leq \left(\int_{E_1} \frac {1}{ G\left(\phi(0)-C_\phi^*(\eta^*(\bx'))\right) } d\gamma_1^*\right)^{\frac 12}\left(\int_{E_1} G\big(S_\e(\phi\circ f)(\bx)-\phi(f(\bx'))\big)d\gamma_1^*\right)^{\frac 12} \\
        &\leq \left(\int\frac {1}{ G\left(\phi(0)-C_\phi^*(\eta^*(\bx'))\right)} d\gamma_1^*\right)^{\frac 12}G\left(\int_{E_1} S_\e(\phi\circ f)(\bx)-\phi(f(\bx'))d\gamma_1^*\right)^{\frac 12},
    \end{aligned}      
    \end{equation}
which implies \cref{eq:conv_goal}.

    \end{enumerate}
\end{proof}

%% file: Appendices/G-1-r.tex
\section{Technical Integral Lemmas}
In this section, we require several technical facts about Riemann–Stieltjes integrals, which we briefly review here.

Let $g:\Rset\to \Rset$, $h:\Rset\to \Rset$ be functions and let $P=\{z_0,z_1,\ldots, z_K\}$ be a partition of an interval $I$. Then the \emph{lower} and \emph{upper} sums with respect to $g,h,P$ are defined as 
\[L(g,h,P)= \sum_{k=0}^{K-1}\inf_{z\in [z_k,z_{k+1}]} g(z)(h(z_{k+1})-h(z_k)),\quad  U(g,h,P)=\sum_{k=0}^{K-1}\sup_{z\in [z_k,z_{k+1}]} g(z)(h(z_{k+1})-h(z_k)) \]
respectively.
When $g$ is non-increasing, these simplify as $\inf_{z\in [z_k,z_{k+1}]} g(z)=g(z_{k+1})$ and $\sup_{z\in [z_k,z_{k+1}]} g(z)=g(z_k)$.

Riemann–Stieltjes integral $\int_I gdh$ can be approximated by upper and lower sums, much as in the classical Riemann case. The following result records the relevant approximation property:

\begin{proposition}\label{prop:R-S_approximation}
    Let $\int_I g dh$ be a Riemann-Stieltjes integral.
    If $g$ is continuous and $h$ is monotone, then the integral exists. Moreover, for any partition $P$, $L(g,h,P)\leq \int_I gdh\leq U(g,h,P)$. In addition, for any $\delta>0$, there exists a partition $P$ for which $U(g,h,P)-\delta\leq \int_I gdh\leq L(g,h,P)+\delta$.
\end{proposition}
For details, see \citep[Theorem~7.17]{Apostol74} or Theorem~2.24 of \citep{WheedenZygmund} for the existence statement and \citep[Theorem~7.27]{Apostol74} for a discussion of upper and lower integrals.

\subsection{The Lebesgue and Riemann–Stieltjes integral of an increasing function}\label{app:lebesge_RS}

The goal of this section is to prove \cref{eq:lebesgue_to_rs}, or namely: 
\begin{proposition}\label{prop:ls_to_rs_decreasing}
Let $f$ be a non-increasing, non-negative, continuous function on an interval $[a,b]$ and let $\QQ$ be a finite positive measure. Let $z$ be a random variable distributed according to $\QQ$ and define $h(\alpha)=\QQ(z\leq \alpha)$. Then
\[\int_{(a,b]} f(z)d\QQ(z)=\int_a^b f(\alpha)dh(\alpha)\]
where the integral on the left is defined as the Lebesgue integral in terms of the measure $\QQ$ while the integral on the right is defined as a Riemann–Stieltjes integral.
\end{proposition}\label{prop:leb_to_rs}
\begin{proof}
    Recall that when $f$ is monotonic, the Riemann-Stieltjes integral is the value of the limits
    \begin{equation}\label{eq:r-s_integral_def}
        \int f dh=\lim_{\Delta \alpha_i\to 0} \sum_{i=0}^{I-1} f(\alpha_i)(h(\alpha_{i+1})-h(\alpha_i))=\lim_{\Delta\alpha_i\to 0} \sum_{i=0}^{I-1} f(\alpha_{i+1})(h(\alpha_{i+1})-h(\alpha_i)),
    \end{equation}
    
    where these limits are evaluated as the size of the partition $\Delta \alpha_i=\alpha_{i+1}-\alpha_i$ approaches 0 \citep[Exercise~7.3, Theorem~7.19]{Apostol74}, 
    while the Lebesgue integral $\int fd\QQ$ is defined as 
    \[\int fd\QQ=\sup \left\{\int gd\QQ:g \leq f, g\text{ simple function, }\right\}.\]

    The limits in \cref{eq:r-s_integral_def} are upper and lower sums because $f$ is monotonic, and thus by \cref{prop:R-S_approximation}, for any $\delta>0$,
  one can choose a partition $\{\alpha_i\}_{i=0}^I$ for which each of the sums in \cref{eq:r-s_integral_def} is within $\delta$ of $\int fdh$.

    Next, consider two simple functions $g_1$, $g_2$ defined according to
    \[g_1(z)=\sum_{i=0}^{I-1} f(\alpha_{i+1})\chi_{z\in (\alpha_i,\alpha_{i+1}]}, \quad g_2(z)=\sum_{i=0}^{I-1} f(\alpha_i)\chi_{z\in (\alpha_i,\alpha_{i+1}]}.\]
    By construction, $g_1(x)\leq f(x)\leq g_2(x)$ for all $x\in(a,b]$. Moreover, since $f(\alpha_i)-f(\alpha_{i+1})<\delta$, it follows that $ f(x)\leq g_2(x)+\delta$ when $x\in (a,b]$. Now applying the definition of the integral of a simple function, we obtain:
    \[\begin{aligned}
        \int fdh-\delta\leq \sum_{i=0}^{I-1} f(\alpha_{i+1})\big(h(\alpha_{i+1})-h(\alpha_i)\big)=\int_{(a,b]} g_1d\QQ\leq \int_{(a,b]} fd\QQ\leq \int_{(a,b]} g_2 d\QQ\\=\sum_{i=0}^{I-1} f(\alpha_{i})\big(h(\alpha_{i+1})-h(\alpha_i)\big)\leq \int fdh+\delta
    \end{aligned} \]
    As $\delta$ is arbitrary, it follows that $\int fdh=\int fd\QQ$.
\end{proof}
 Notice that because $H(0)=0$, the integral in the right-hand side of \cref{eq:lebesgue_to_rs} is technically an improper integral. Thus to show \cref{eq:lebesgue_to_rs}, one can conclude that 
 \[\int_{z\in (\delta,1/2]} \frac 1 {H(z)} d\QQ(z)=\int_\delta^{1/2} \frac 1 {H(\alpha)} dh(\alpha)\]
 from \cref{prop:ls_to_rs_decreasing}
  and then take the limit $\delta\to 0$.

\subsection{Proof of the last equality in \cref{eq:1-r_int}}\label{app:1-r_int}

The goal of this appendix is to prove the following inequality:
\begin{lemma}\label{lemma:1-r}
Let $h:[0,1/2]\to [0,1]$ be an increasing and right-continuous function with $h(0)=0$ and $h(1/2)\leq 1$. Let $H$ be any continuous function with $H\geq h$ and let $r\in [0,1)$. Then one can bound the Riemann–Stieltjes integral $\int 1/H(z)^r dh$ by
\[\int_0^{1/2} \frac 1 {H(z)^r}dh\leq \frac {1} {1-r}\]
\end{lemma}
\begin{proof}
    Let $\delta>0$, then one can pick a partition $P=\{z_0=0,z_1,\ldots, z_K=1/2\}$ for which $\int_0^{1/2} H^{-r}dh\leq L(H^{-r},h ,P)+\delta$. As $H^{-r}$ is non-increasing, $L(H^{-r},h,P)=\sum_{k=0}^{K-1} H^{-r}(z_{k+1})(h(z_{k+1})-h(z_k))$. Therefore, if we define $a_k=h(z_k)$, then 
    \begin{equation}\label{eq:integ_evalutaion}
       \begin{aligned} &\int_0^{1/2} H^{-r}dh\leq \sum_{k=0}^{K-1} H^{-r}(z_{k+1})(h(z_{k+1})-h(z_k))+\delta\leq \sum_{k=1}^{K-1} h^{-r}(z_{k+1})(h(z_{k+1})-h(z_k))+\delta\\
       &= \sum_{k=0}^{K-1} a_{k+1}^{-r}(a_{k+1}-a_k)+\delta
       \end{aligned}
    \end{equation}

    Because the function $y\mapsto y^{-r}$ is decreasing in $y$, one can bound $a_{k+1}^{-r} (a_{k+1}-a_k)\leq \int_{a_k}^{a_{k+1}} y^{-r}d y$ and consequently the sum in \cref{eq:integ_evalutaion} is bounded above as
    \[ \sum_{k=0}^{K-1}\int_{a_k}^{a_{k+1}} y^{-r} dy=\int_0^{h(1/2)} y^{-r} dy\leq \int_0^1 y^{-r}dy=\frac 1 {1-r}\]
    Therefore $\int_0^{1/2} H^{-r}dh\leq 1/(1-r)+\delta$. The result follows as $\delta>0$ is arbitrary. 
\end{proof}

%% file: Appendices/H_optimizing_bound.tex
\section{Optimizing the Bound of \cref{lemma:multiple_rs_bound} over $r$}\label{app:optimizing_bound}

\begin{proof}[Proof of \cref{thm:convex_surrogate_bound}]
    Let 
    \[f(r)= \frac 1{1-r} a^r\]
    Then
    \[f'(r)= \frac 1{(1-r)^2} a^r+\frac 1{1-r} \ln a a^r\]
    solving $f'(r^*)=0$ produces $r^*= 1+\frac 1 {\ln a}$, and 
    \[f\left(1+ \frac 1 {\ln a}\right)= -\ln a a^{1+\frac 1 {\ln a}}=-e a \ln a\]
    One can verify that this point is a minimum via the second derivative test: 
    \[f'(r)=\left(\frac 1 {1-r}+\ln a\right)f(r)\]
    and thus 
    \[f''(r)=\left(\frac 1 {1-r}+\ln a\right)f'(r)+  \frac 1 {(1-r)^2} f(r).\]
    Consequently, $f''(r^*)=\ln(a)^2f(1+\frac 1 {\ln a})>0$.
    
    However, the point $r^*$ is in the interval $[0,1]$ only when $a\in [0,e^{-1}]$. When $a>e^{-1}$, $f$ is minimized over $[0,1]$ at $r=0$. Because $r^*$ is a minimizer when $a\in[0,e^{-1}]$, one can bound $f(0) \geq f(r^*)$ over this set and thus
    \[f(r)\leq \min\left(1,-e a\ln a\right)\]

\end{proof}

%% file: Appendices/I-non_consistent_deferred.tex
\section{Deferred proofs from \cref{sec:non_consistent}}
\subsection{Existence of Minimizers and Complementary slackness}

The existence and complimentary slackness theorems of \cref{app:maximizer_comparison} extend to $\prm$.
Observe that minimizers of $R_\phi$ may assume values in $\ov \Rset$; for example, with the exponential loss $\phi(\alpha) = e^{-\alpha}$ and the distribution defined by $\eta(\bx) \equiv 1$, the unique minimizer of $R_\phi$ is $+\infty$. 
Just as in the non-adversarial scenario, $\prm$ may fail to attain its infimum over $\Rset$-valued functions. Nevertheless, \citep[Lemma~8]{FrankNilesWeed23consistency} and \citep[Theorem~6]{frank2024consistency} guarantee the existence of a minimizer over $\ov \Rset$-valued functions.

\begin{theorem}\label{th:existence_surrogate_primal}
Let $\phi$ satisfy \cref{as:phi}. Then
    \[\inf_{f \text{ $ \Rset$-valued }} \prm (f)=\inf_{f \text{ $\ov \Rset$-valued}} \prm(f).\]
    Furthermore, equality is attained at some Borel measurable, $\ov \Rset$-valued function $f^*$.
    
\end{theorem}

 Moreover, Theorem~7 of \citep{FrankNilesWeed23minimax} describes two conditions that characterize minimizers of $\prm$ and maximizers $\dl$. 
\begin{theorem}[Complementary Slackness]\label{thm:comlementary_slackness}
The function $f^*$ minimizes
$\prm$ and the measures $(\PP_0^*, 
\PP_1^*)$ maximize $\dl$ over 
$\Wball \e(\PP_0)\times \Wball \e(\PP_1)$ 
iff the following two conditions hold:
\begin{enumerate}[label=\arabic*)]
    \item\label{it:sup} \[\int S_\e(\phi(f^*))d\PP_1=\int \phi(f^*)d\PP_1^*\quad \text{and}\quad \int S_\e(\phi(-f^*))d\PP_0=\int\phi(-f^*)d\PP_0^*\]
    \item \label{it:eq_comp_slack}\[C_\phi(\eta^*,f^*)=C_\phi^*(\eta^*)\quad \PP^*\text{-a.e.}\]
\end{enumerate}
\end{theorem}

\subsection{Proof of \cref{lemma:restrict_dual,lemma:optimal_dual_split}}
\label{app:non_consisten_lemmas} 

As a preliminary step, we establish that if $\PP_0^*$, $\PP_1^*$ are induced by transport maps, then these maps determine the locations of maximizers of $\phi \circ f$ and $\phi \circ -f$.

\begin{lemma}\label{lemma:transport_sup}
    Let $\PP_0^*,\PP_1^*$ be maximizers of $\dl$ induced by the transport maps $T_0,T_1$ satisfying $\|T_0(\bx)-\bx\|\leq \e$, $\|T_1(\bx)-\bx\|\leq \e$. Then any minimizer $f^*$ of $\prm$ satisfies 
    
    \noindent\begin{minipage}{0.5\linewidth}
    \begin{equation}\label{eq:transport_sup_0}
       S_\e(\phi\circ -f^*)(\bx)=\phi(-f^*(T_0(\bx))\quad \PP_0\text{-a.e.}
    \end{equation}
    \end{minipage}%
    \begin{minipage}{0.5\linewidth}
        \begin{equation}\label{eq:transport_sup_1}
         S_\e(\phi\circ f^*)(\bx)=\phi(f^*(T_1(\bx))\quad \PP_1\text{-a.e.}
    \end{equation}
    \end{minipage}

\end{lemma}
\begin{proof}
    We show \cref{eq:transport_sup_1}, the argument for \cref{eq:transport_sup_0} is analogous.

        Let $f^*$ minimize $\prm$; such a function exists by \cref{th:existence_surrogate_primal}. 
    The complementary slackness condition \cref{it:sup} in \cref{thm:comlementary_slackness} yields 
    \[\int S_\e(\phi\circ f^*)(\bx)d\PP_1=\int \phi \circ f^*(\bx') d\PP_1^*=\int \phi(f^*(T_1(\bx)))d\PP\] 
    As the relation $\|T_1(\bx)-\bx\|\leq \e$ implies $S_\e(\phi\circ f^*)(\bx)\geq \phi(f^*(T_1(\bx)))$ one can conclude \cref{eq:transport_sup_1}.
\end{proof}

Next, we verify strong duality for these restricted measures, utilizing the notation defined in \cref{eq:augmented_primal,eq:augmented_optimal}. The statement below implies \cref{lemma:restrict_dual}.
\begin{lemma}\label{lemma:strong_duality_verification}
    Let $\PP_0,\PP_1,\PP_0^*,\PP_1^*,T_0,T_1,U_0,U_1$ and $Q$ be as in \cref{lemma:restrict_dual}, and let $f^*$ minimize $\prm(\,\cdot\,; \PP_0,\PP_1)$. Then 
    \begin{equation}\label{eq:ball_containment}
        \PP_0^*|_{Q}\in \Wball \e (\PP_0|_{U_0}), \quad  \PP_1^*|_{Q}\in \Wball \e (\PP_1|_{U_1})
    \end{equation}
        \[\prm(f^*; \PP_0|_{U_0}, \PP_1|_{U_1})=\dl(\PP_0^*|_{Q}, \PP_1^*|_{Q}).\]
        Consequently $f^*$ minimizes $\prm(\,\cdot\,; \PP_0|_{U_0}, \PP_1|_{U_1})$ while $\PP_0^*|_{Q}$, $\PP_1^*|_{Q}$ maximize $\dl$ over $\Wball \e (\PP_0|_{U_0})\times\Wball \e (\PP_1|_{U_1})$.
\end{lemma}

\begin{proof}
         By the definitions of $Q$, $U_0$, $U_1$ \[\PP_0^*|_{Q}=\PP_0|_{U_0}\sharp T_0,\quad \PP_1^*|_{Q}=\PP_1|_{U_1}\sharp T_1.\] The relations $\|T_0(\bx)-\bx\|\leq \e$, $\|T_1(\bx)-\bx\|\leq \e$  imply \cref{eq:ball_containment}. Next, let $\tilde \eta= d\PP_1^*|_Q/d(\PP_1^*|_Q+\PP_0^*|_Q)$. On the set $Q$,
     \begin{equation}\label{eq:eta_comparison}
        \tilde \eta=\eta^*\quad \PP^*\text{-a.e.}
    \end{equation}

    Next, let $f^*$ be a minimizer of $\prm(\cdot; \PP_0,\PP_1)$. \Cref{lemma:transport_sup} and the definitions of $T_0$, $T_1$ imply that 
    \begin{align}
        \prm(f^*; \PP_0|_{U_0}, \PP_1|_{U_1})&= \int S_\e(\phi\circ f^*)(\bx) \one_{U_1}(\bx)d\PP_1+\int S_\e(\phi\circ -f^*)(\bx)\one_{U_0}(\bx)d\PP_0\\
        &\int \phi(f^*(T_1(\bx)))\one_{Q}(T_1(\bx))d\PP_1+\int \phi(-f^*(T_0(\bx))\one_Q(T_0(\bx)) d\PP_0 \\
        &=\int \phi(f^*(\bx'))\one_Q(\bx')d\PP_1\sharp T_1+\int \phi(-f^*(\bx'))\one_Q(\bx')d\PP_0\sharp T_0\\
        &=\int C_\phi(\eta^*,f^*)\one_Q d\PP^*
    \end{align}
    Since $f^*$ minimizes $\prm(f; \PP_0, \PP_1)$, the complimentary slackness condition \cref{it:eq_comp_slack} of \cref{thm:comlementary_slackness} implies $C_\phi(\eta^*,f^*)=C_\phi^*(\eta^*)$. \Cref{eq:eta_comparison} further implies $C_\phi^*(\eta^*)=C_\phi^*(\tilde \eta)$ on $Q$ $\PP$-a.e. and therefore,
    \[\int C_\phi(\eta^*,f^*)\one_Q d\PP^*=\int C_\phi^*(\td \eta) d\PP^*|_Q=\cdl(\PP_0^*|_Q,\PP_1^*|_Q)\]
\end{proof}

These results show that restricted measures in the dual corresponds directly to restricted measures in the primal.

\begin{proof}[Proof of \cref{lemma:restrict_dual}]
         Applying \Cref{th:strong_duality_surrogate} to the restricted measures $\PP_1|_{U_1}$, $\PP_0|_{U_0}$ and invoking \cref{lemma:strong_duality_verification} yields the claim.
\end{proof}

Finally, one can conclude \cref{lemma:optimal_dual_split} by comparing $\prm(f^*;\PP_0|_{T_0^{-1}(Q^C)},\PP_1|_{T_1^{-1}( Q^C)})$ and $\prm(f^*;\PP_0|_{ U_0^C},\PP_1|_{ U_1^C})$.
\begin{proof}[Proof of \cref{lemma:optimal_dual_split}]
 Observe that $U_0^C=T_0^{-1}(Q^C)$, $U_1^C=T_1^{-1}(Q^C)$ and let $f^*$ be a minimizer of $\prm(\,\cdot\,;\PP_0,\PP_1)$. Then \cref{lemma:strong_duality_verification,lemma:restrict_dual} imply
 
    \begin{align}
        &R_{\phi,*}^\e(\PP_0|_{U_0},\PP_1|_{U_1})=\prm(f^*;\PP_0|_{U_0},\PP_1|_{U_1})=\int_Q C_\phi^*(\eta^*)d\PP^*\nonumber\\
        &R_{\phi,*}^\e(\PP_0|_{U_0^C},\PP_1|_{ U_1^C})=\prm(f^*;\PP_0|_{ U_0^C},\PP_1|_{U_1^C})=\int_{Q^C} C_\phi^*(\eta^*)d\PP^*\label{eq:optimal_restrict_second}
    \end{align}
 
 Summing these:

     \begin{align*}
            \prm(f^*;\PP_0|_{U_0},\PP_1|_{U_1})&+\prm(f^*;\PP_0|_{\td U_0},\PP_1|_{\td U_1})= \int C_\phi^*(\eta^*)d\PP^*=R_{\phi,*}^\e(f^*;\PP_0,\PP_1)  
     \end{align*}

\end{proof}

\subsection{Proof of \cref{thm:non_consistent_bound_linear}}\label{app:non_consistent_bound_linear}
\begin{proof}[Proof of \cref{thm:non_consistent_bound_linear}]
    Let $Q=\{\bx': |\eta^*(\bx')-1/2|<\alpha\}$. Then \cref{lemma:restrict_dual} applied to $Q^C$ shows that $(\PP_0^*|_{Q^C},\PP_1^*|_{Q^C})$ maximize $\dl$ over $\Wball \e (\PP_0|_{U_0})\times \Wball \e(\PP_1|_{U_1})$, with $U_0=T_0^{-1}(Q^C)$ and $U_1=T_1^{-1}(Q^C)$. Next, observe that simply scaling the inequality \cref{eq:adv_surrogate_massart_ineq} shows that \cref{thm:adv_surrogate_massart} applies even when $\PP(\Rset^d)\leq1$. Consequently, \cref{thm:adv_surrogate_massart,lemma:optimal_dual_split} imply that
    \begin{align*}
        \cprm(f;\PP_0|_{U_0},\PP_1|_{U_1})-R_*^\e(\PP_0|_{U_0},\PP_1|_{U_1})&\leq \const{\alpha} \left(\prm(f;\PP_0|_{U_0},\PP_1|_{U_1})-R_{\phi*}^\e(\PP_0|_{U_0},\PP_1|_{U_1})\right)\\
        &\leq  \const{\alpha} \left(\prm(f;\PP_0,\PP_1)-R_{\phi*}^\e(\PP_0,\PP_1)\right)
    \end{align*}
    Next, by \cref{lemma:optimal_dual_split}, adding $\cprm(f;\PP_0|_{U_0^C},\PP_1|_{U_1^C})- R_*^\e(\PP_0|_{U_0^C}, \PP_1|_{U_1^C})$ to both sides of the inequality above results in 
    \[        \begin{aligned}
        &\cprm(f;\PP_0,\PP_1)-R_*^\e(\PP_0,\PP_1)\leq  \\
        &\frac 1 {\phi(0)- C_\phi^*(1/2-\alpha)}\left(\prm(f;\PP_0,\PP_1)-R_{\phi*}^\e(\PP_0|_{U_0},\PP_1)\right)+\cprm(f;\PP_0|_{U_0^C},\PP_1|_{U_1^C})- \int_QC^*(\eta^*)d\PP^*
    \end{aligned}\]
    The fact that $C^*(\eta^*)\geq 1/2-\alpha$ on $Q$ while $S_\e(\one_{f\leq0})\leq 1$, $S_\e(\one_{f>0})\leq 1$ implies that $\cprm(f;\PP_0|_{U_0^C},\PP_1|_{U_1^C})- \int_QC^*(\eta^*)d\PP^*\leq (\frac 12 +\alpha)\PP^*(|\eta-1/2|<\alpha) $. Thus, the excess risk contribution from the region $A$ is at most $(1/2+\alpha)\PP^*(|\eta^*-1/2|<\alpha)$.
\end{proof}

%% file: Appendices/J-gaussian_and_massart.tex
\section{Further details from \cref{ex:realizable,ex:gaussian,ex:massart}}\label{app:gaussian_details}

In \cref{app:dual_attack_massart,app:dual_attack}, we use an operation analogous to $S_\e$ that calculates the infimum of a function over an $\e$-ball. Formally, we define: 
\begin{equation}\label{eq:I_e_def}
    I_\e(g)(\bx)=\inf_{\|\bx'-\bx\|\leq \e} g(\bx').    
\end{equation}

Next, we define a mapping $\alpha_\phi$ from $\eta\in[0,1]$ to minimizers of $C_\phi(\eta,\cdot)$ by 

\begin{equation}
    \label{eq:alpha_phi_def}
    \alpha_\phi(\eta)=\inf\{\alpha:\alpha\text{ is a minimizer of }C_\phi(\eta,\cdot)\}.
\end{equation}
Lemma~25 of \citep{FrankNilesWeed23minimax} shows that the function $\alpha_\phi$ defined in \cref{eq:alpha_phi_def} maps $\eta$ to the smallest minimizer of $C_\phi(\eta,\cdot)$ and is non-decreasing. This property will be instrumental in constructing minimizers for $\prm$.

\subsection{Proof of \cref{lemma:realizable_eta}}\label{app:realizable_eta}

\begin{proof}[Proof of \cref{lemma:realizable_eta}]
            If $R_*^\e=0$, by \cref{th:strong_duality_classification}, for any measures $\PP_0'\in \Wball \e(\PP_0)$, $\PP_1'\in \Wball \e(\PP_1)$ we have
    $\PP'(\eta'=0\text{ or }1)=1$, where $\PP'=\PP_0'+\PP_1'$ and $\eta'=d\PP_1'/d\PP'$. This statement must also hold for the $\PP_0^*\in \Wball \e(\PP_0)$, $\PP_1^*\in \Wball \e (\PP_1)$ that maximize $\dl$.
\end{proof}

\subsection{Calculating the optimal $\PP_0^*$, $\PP_1^*$ for \cref{ex:massart}}\label{app:dual_attack_massart}

First, notice that a minimizer of $R_\phi$ is given by $f(x)=\alpha_\phi(\eta(x))$ with $\eta(x)$ as defined in \cref{eq:eta_massart_example}.
Below, we construct a minimizer $f^*$ for $\prm$. We'll do this construction separately for $\e\leq \delta$ and $\e\in (\delta, 1-\delta)$. 
\subsubsection*{When $\e\leq \delta$:}
Define a function $\tilde \eta:[-\delta-\e-1,1+\delta+\e]\to [0,1]$ by 

\[\tilde \eta(x)=\begin{cases}
    \frac 14&\text{if }x\in[-1-\delta-\e,0)\\
    \frac 12&\text{if }x=0\\
    \frac 34&\text{if }x\in (0,1+\delta+\e]
\end{cases}\]

and a function $f^*$ by $f^*(x)=\alpha_\phi(\tilde
\eta(x))$. 

We'll verify that this function is a minimizer by showing that $\prm(f^*)=R_\phi(f)$. As the minimal possible adversarial risk is bounded below by $R_{\phi,*}$, one can conclude that $f^*$ minimizes $\prm$. Consequently, $\dl(\PP_0,\PP_1)=\prm(f)$ and thus the strong duality result in \cref{th:strong_duality_surrogate} would imply that $\PP_0$, $\PP_1$ must maximize the dual problem.

As both $\tilde \eta$ and $\alpha_\phi$ are non-decreasing, the function $f^*$ must be non-decreasing as well. Consequently, $S_\e(\phi(f^*))(x)=\phi(I_\e(f^*)(x))=\phi(f^*(x-\e))$ and similarly, $S_\e(\phi(-f^*))(x)=\phi(-S_\e(f^*)(x))=\phi(-f^*(x+\e))$. (Recall the $I_\e$ operation was defined in \cref{eq:I_e_def}.)

Therefore,
\begin{equation}\label{eq:sup_to_dec}
\begin{aligned}
    \prm(f^*)&=\int S_\e(\phi(f^*))(x)p_1(x)dx+\int S_\e(\phi(-f^*))(x)p_0(x)dx=\int \phi(f^*(x-\e)) p_1(x)dx+\int \phi(-f^*(x+\e))p_0(x)dx\\
    &=\int \phi(f^*(x))p_1(x+\e)dx+\int \phi(-f^*(x))p_0(x-\e)dx
\end{aligned}
\end{equation}

Consequently,
\begin{align*}
\int \phi(f^*(x))p_1(x+\e)dx=&\int_{-1-\delta-\e}^{-\delta-\e} \frac 18 \phi\Big(\alpha_\phi\Big(\frac 14\Big)\Big)dx+\int_{\delta-\e}^{1+\delta-\e} \frac 18 \phi\Big(\alpha_\phi\Big(\frac 34\Big)\Big) dx\\
&= \int_{-1-\delta}^{-\delta} \frac 18 \phi\Big(\alpha_\phi\Big(\frac 14\Big)\Big)dx+\int_{\delta}^{1+\delta} \frac 18 \phi\Big(\alpha_\phi\Big(\frac 34\Big)\Big) dx=\int \phi(f(x))p_1(x)dx    
\end{align*}

Analogously, one can show that 
\[\int \phi(-f^*(x))p_0(x-\e)dx=\int \phi(-f(x))p_0(x)dx\]
and consequently $R_\phi^\e(f^*)=R_\phi(f)$.

\subsubsection*{When $\e \in (\delta,1+\delta)$:}
We will show that $\prm(f^*)=\dl(\PP_0^*,\PP_1^*)$ for the proposed attacks, proving that $\PP_0^*$, $\PP_1^*$ are dual optimal distributions. This time, define the function $\tilde \eta:[-\delta-\e-1,1+\delta+\e]\to [0,1]$ by

\[\tilde \eta(x)=\begin{cases}
    0&\text{if}x\in [-1-\delta-\e,-1-\delta+\e)\\
    \frac 14&\text{if }x\in[-1-\delta+\e,-(\e-\delta))\\
    \frac 12&\text{if }x\in[-(\e-\delta),(\e-\delta)]\\
    \frac 34&\text{if }x\in ((\e-\delta),1+\delta-\e]\\
    1 &\text{if }x\in (1+\delta-\e,1+\delta+\e]
\end{cases}\]
and again take $f^*(x)=\alpha_\phi(\tilde \eta(\bx))$. The function $f^*$ is non-decreasing, so again \cref{eq:sup_to_dec} holds. Further, defining $p_1^*$, $p_0^*$ as $p_1^*(x)=p_1(x+\e)$ and $p_0^*(x)=p_0(x-\e)$ implies the relation
\[\prm(f^*)=\int C_\phi(\eta^*,f^*)p^*(x)dx,\]
where $p^*(x)=p_0^*(x)+p_1^*(x)$ and $\eta^*=p_1^*(x)/p^*(x)$. The function $\tilde \eta$ was defined so that $\tilde \eta(x)=\eta^*(x)$ a.e. and hence 
\[C_\phi(\eta^*,f^*)=C_\phi(\eta^*,\alpha_\phi(\eta^*))=C_\phi^*(\eta^*).\]
This relation implies $\prm(f^*)=\dl(\PP_0^*,\PP_1^*)$, where $\PP_0^*$, $\PP_1^*$ are the distributions with pdfs $p_0^*$ and $p_1^*$.

\subsection{Calculating the optimal $\PP_0^*$ and $\PP_1^*$ for \cref{ex:gaussian}--- Proof of \cref{eq:dual_optimal_attacks}}\label{app:dual_attack}
We will show that the densities in \cref{eq:dual_optimal_attacks} are dual optimal by finding a function $f^*$ for which $\prm(f^*)=\dl(\PP_0^*,\PP_1^*)$. \Cref{th:strong_duality_surrogate} will then imply that $\PP_0^*$, $\PP_1^*$ must maximize the dual. Define $\eta^*(x)$ by
\[\eta^*(x)=\frac{p_1^*(x)}{p_1^*(x)+p_0^*(x)},\]
with $p_0^*(x)$ and $p_1^*(x)$ as in \cref{eq:dual_optimal_attacks}.
For a given loss $\phi$ we will prove that the optimal function $f^*$ is given by 
\[f^*(x)=\alpha_\phi(\eta^*(x)).\]

The function $\eta^*$ computes to 

\[\eta^*(x)=\frac 1 {1+e^{\frac {\mu_1-\mu_0-2\e} {\sigma^2}(\frac{\mu_1+\mu_0}2-x)}}.\]
If $\mu_1-\mu_0>2\e$, the conditional probability $\eta^*(x)$ is increasing in $x$ and consequently the function $f^*$ is non-decreasing. Therefore, $S_\e(\phi(f^*))(x)=\phi(I_\e(f^*)(x))=\phi(f^*(x-\e))$ (recall $I_\e$ was defined in \cref{eq:I_e_def}). Similarly, one can argue that $S_\e(\phi(-f^*))(x)=\phi(-f^*(x+\e))$, and therefore,
\begin{align*}
    \prm(f^*)&=\int S_\e(\phi(f^*))(x)p_1(x)dx+\int S_\e(\phi(-f^*))(x)p_0(x)dx=\int \phi(f^*(x-\e)) p_1(x)dx+\int \phi(-f^*(x+\e))p_0(x)dx\\
    &=\int \phi(f^*(x))p_1(x+\e)dx+\int \phi(-f^*(x))p_0(x-\e)dx.
\end{align*}

Next, notice that $p_1(x+\e)=p_1^*(x)$ and $p_0(x-\e)=p_0^*(x)$. Define $\PP^*=\PP_0^*+\PP_1^*$. Then 
\begin{align*}
    \prm(f^*)&=\int \eta^*\phi(\alpha_\phi(\eta^*))+(1-\eta^*)\phi(-\alpha_\phi(\eta^*))d\PP^*=\int C_\phi^*(\eta^*)d\PP^*=\dl(\PP_0^*,\PP_1^*)
\end{align*}
Consequently, the strong duality result in \cref{th:strong_duality_surrogate} implies that $\PP_0^*$ $\PP_1^*$ must maximize the dual $\dl$.

\subsection{Showing \cref{eq:concavity_bound}}\label{app:gaussian_linear_bound}
\begin{lemma}\label{lemma:eta_calculation}
    Consider an equal gaussian mixture with variance $\sigma$ and means $\mu_0<\mu_1$, with pdfs given by 
    \[p_0(x)=\frac 12 \cdot\frac 1 {\sqrt{2\pi}\sigma} e^{-\frac{(x-\mu_0)^2}{2\sigma^2}},\quad p_1(x)=\frac 12\cdot \frac 1 {\sqrt{2\pi}\sigma}e^{-\frac{(x-\mu_1)^2}{2\sigma^2}}\]
    Let $\eta(x)=p_1(x)/(p_0(x)+p_1(x))$. Then $|\eta(x)-1/2|\leq z$ iff $x\in [\frac{\mu_0+\mu_1}2 -\Delta(z), \frac{\mu_0+\mu_1}2+\Delta(z)]$, where $\Delta(z)$  is defined by 
    \begin{equation}
        \label{eq:Delta_z_def}
        \Delta(z)=\frac {\sigma^2}{\mu_1-\mu_0}\ln\left(\frac {\frac 12+z} {\frac 12 -z}\right).    
    \end{equation}
\end{lemma}
\begin{proof}
    The function $\eta$ can be rewritten as $\eta(x)=1/(1+p_0/p_1)$ while
    \[\frac{p_0(x)}{p_1(x)}=\exp\left(-\frac{(x-\mu_0)^2}{2\sigma^2}+\frac{(x-\mu_1)^2}{2\sigma^2} \right)=\exp\left( \frac{\mu_1-\mu_0}{\sigma^2}\left(\frac{\mu_1+\mu_0}2-x\right)\right)\]

    Consequently, $|\eta(x)-1/2|\leq z$ is equivalent to 
    \[\frac 12-z \leq \frac 1 {\exp\left(\frac{\mu_1-\mu_0}{\sigma^2} (\frac{\mu_1+\mu_0} 2 -x)\right)+1}\leq \frac 12+z\]

    which is equivalent to 
    \[\frac{\mu_1+\mu_0}2 -\frac {\sigma^2}{\mu_1-\mu_0}\ln\left(\frac 1 {\frac 12 -z} -1\right)\leq x \leq \frac{\mu_1+\mu_0}2-\frac {\sigma^2}{\mu_1-\mu_0}\ln\left( \frac 1 {z+\frac 12} -1\right)\]
    Finally, notice that 
    \begin{equation}
        \label{eq:Delta_z_fla}
        \Delta(z)=\frac{\sigma^2}{\mu_1-\mu_0} \ln\left( \frac 1 {\frac 1 2 -z} -1\right)
    \end{equation}
    while
    \[-\Delta(z)=\frac {\sigma^2}{\mu_1-\mu_0}\ln\left( \frac 1 {z+\frac 12} -1\right)\]
\end{proof}

\begin{lemma}
     Let $p_0,p_1$, and $\eta$ be as in \cref{lemma:eta_calculation} and let 
    $h(z)=\PP(|\eta-1/2|\leq z)$. Then if $\mu_1-\mu_0\leq \sqrt 2 \sigma$, then $h$ is concave. 
\end{lemma}
\begin{proof}
    To start, we calculate the second derivative of $\Delta(z)$ and the first derivative of $p_0$.

     The first derivative of $\Delta$ is 
     \[\Delta'(z)=\frac{\sigma^2}{\mu_1-\mu_0}\cdot \frac 1 {\frac 1 4-z^2}.\]
     and the second derivative of $\Delta(z)$ is
    \begin{equation}
        \label{eq:Delta_z_second_derivative}
        \Delta''(z)=\frac{\sigma^2}{\mu_1-\mu_0}\cdot \frac{2z}{(\frac 14 -z^2)^2}
    \end{equation}

Next, one can calculate the derivative of $p_0$ as 
\begin{equation}
    \label{eq:p_0_der}
    p_0'(x)=\frac 12 \cdot \frac 1 {\sqrt{2\pi}\sigma}\cdot \frac {-(x-\mu_0)}{\sigma^2} e^{-\frac{(x-\mu_0)^2}{2\sigma^2}}=-\frac{(x-\mu_0)}{\sigma^2}p_0(x) 
\end{equation}
and similarly
\begin{equation}\label{eq:p_1_der}
    p_1'(x)=-\frac{(x-\mu_1)}{\sigma^2}p_1(x)
\end{equation}
    Let $p(x)=p_0+p_1$. 
    \Cref{lemma:eta_calculation} implies that the function $h$ is given by $h(z)=\int_{\tfrac {\mu_1+\mu_0}2 -\Delta(z)}^{\tfrac {\mu_1+\mu_0}2 +\Delta(z)} p(z)dz$.
    The first derivative of $h$ is then
    \[h'(z)=\left(p\Big(\frac{\mu_1+\mu_0}2+\Delta(z)\Big)+p\Big(\frac{\mu_1+\mu_0}2-\Delta(z)\Big)\right)\Delta'(z).\]
    Differentiating $h$ twice results in
    \begin{align}
        h''(z)&=\left(p\Big(\frac{\mu_1+\mu_0}2+\Delta(z)\Big)+p\Big(\frac{\mu_1+\mu_0}2-\Delta(z)\Big)\right)\Delta''(z)\nonumber\\
        &+\left(p'\Big(\frac{\mu_1+\mu_0}2+\Delta(z)\Big)-p'\Big(\frac{\mu_1+\mu_0}2-\Delta(z)\Big)\right)(\Delta'(z))^2\nonumber\\
        &=\left(p\Big(\frac{\mu_1+\mu_0}2+\Delta(z)\Big)+p\Big(\frac{\mu_1+\mu_0}2-\Delta(z)\Big)\right)\left(\Delta''(z)-\frac{\Delta(z)\Delta'(z)^2}{\sigma^2}\right)\label{eq:second_der_first}\\
         &+\bigg(p_1\Big(\frac{\mu_1+\mu_0}2+\Delta(z)\Big)+p_1\Big(\frac{\mu_1+\mu_0}2-\Delta(z)\Big)\bigg)\frac{\mu_1-\mu_0}{2\sigma^2}(\Delta'(z))^2\label{eq:second_der_second_p1}\\
         &-\bigg(p_0\Big(\frac{\mu_1+\mu_0}2+\Delta(z)\Big)+p_0\Big(\frac{\mu_1+\mu_0}2-\Delta(z)\Big)\bigg)\frac{\mu_1-\mu_0}{2\sigma^2}(\Delta'(z))^2.\label{eq:second_der_second_p2}
    \end{align}
where the final equality is a consequence of \cref{eq:p_0_der,eq:p_1_der}.
Next, we'll argue that the sum of the terms in \cref{eq:second_der_second_p1,eq:second_der_second_p2} is zero:
\begin{align*}
    &\bigg(p_1\Big(\frac{\mu_1+\mu_0}2+\Delta(z)\Big)+p_1\Big(\frac{\mu_1+\mu_0}2-\Delta(z)\Big)\bigg)
         -\bigg(p_0\Big(\frac{\mu_1+\mu_0}2+\Delta(z)\Big)+p_0\Big(\frac{\mu_1+\mu_0}2-\Delta(z)\Big)\bigg)\\
         &=\frac 1 {2 \sqrt{2\pi}\sigma} \left( \left(e^{-\frac{\left (\frac{\mu_0-\mu_1}2+\Delta(z)\right)^2}{2\sigma^2}} + e^{-\frac{\left (\frac{\mu_0-\mu_1}2-\Delta(z)\right)^2}{2\sigma^2}}\right) - \left(e^{-\frac{\left (\frac{\mu_1-\mu_0}2+\Delta(z)\right)^2}{2\sigma^2}} + e^{-\frac{\left (\frac{\mu_1-\mu_0}2-\Delta(z)\right)^2}{2\sigma^2}}\right)\right)\\
         &=0
\end{align*}

Next, we'll show that under the assumption $\mu_1-\mu_0\leq \sqrt 2 \sigma$, the term \cref{eq:second_der_first} is always negative. Define $k=\sigma^2/(\mu_1-\mu_0)$. Then
\begin{equation}
    \label{eq:second_der_second_calculation}
    \Delta''(z)-\Delta(z)\frac{\Delta'(z)^2}{\sigma^2}=\frac{2k}{(\frac 14-z^2)^2}\left( z-\frac{k^2}{2\sigma^2}\ln\left( \frac 1 {\frac 12-z} -1\right)\right)
\end{equation}

The fact that $\Delta''(z)> 0$ for all $z$ implies that $\ln(1/(1/2-z) -1)$ is convex, and this function has derivative $4$ at zero. Consequently, $\ln(1/(1/2-z) -1)\geq 4z$ and \cref{eq:second_der_second_calculation} implies 

\[ \Delta''(z)-\Delta(z)\frac{\Delta'(z)^2}{\sigma^2}\leq \frac{2k}{(\frac 14-z^2)^2}(z-\frac{k^2}{2\sigma^2}\cdot 4z)=\frac{2kz}{(\frac 14-z^2)^2}\left(1-\frac{2k^2}{\sigma^2}\right)\]
 The condition $\mu_1-\mu_0\leq \sqrt 2 \sigma$ is equivalent to $1-2k^2/\sigma^2<0$.
    
\end{proof}

This lemma implies that $h(z)\leq h'(0)z$. Noting also that $h(z)\leq 1$ for all $z$ produces the bound
\[h(z)\leq \min\left( \frac {16 \sigma^2}{\mu_1-\mu_0} z ,1\right)\]
applying this bound to the gaussians with densities $p_0^*$ and $p_1^*$ results in \cref{eq:concavity_bound}.